\documentclass[sigconf]{acmart}

\usepackage{booktabs} 
\usepackage{xcolor} 
\usepackage{graphicx}
\usepackage{mathrsfs}

\usepackage[stable]{footmisc}
\usepackage{amsmath,amsfonts,bm,amsthm}
\usepackage{autobreak}
\usepackage{algorithmic,algorithm}
 \usepackage{url}
\usepackage{textcomp}
\usepackage{epstopdf}
\usepackage{url}
\usepackage{subfigure}
\newtheorem{remark}{Remark}

\newenvironment{proof-idea}{\renewcommand{\proofname}{Proof Sketch}\proof}{\endproof}

\begin{document}
\fancyhead{}
\title{Towards Plausible Differentially Private ADMM Based Distributed Machine Learning}

\author{Jiahao Ding}
\affiliation{%
 \institution{University of Houston}
}
\email{jding7@uh.edu}

\author{Jingyi Wang}
\affiliation{%
 \institution{San Francisco State University}
}
\email{jingyiwang@sfsu.edu}

\author{Guannan Liang}
\affiliation{%
 \institution{University of Connecticut}
}
\email{guannan.liang@uconn.edu}

\author{Jinbo Bi}
\affiliation{%
 \institution{University of Connecticut}
}
\email{jinbo.bi@uconn.edu}

\author{Miao Pan}
\affiliation{%
 \institution{University of Houston}
}
\email{mpan2@uh.edu}

\begin{abstract}
The Alternating Direction Method of Multipliers (ADMM) and its distributed version have been widely used in machine learning. In the iterations of ADMM, model updates using local private data and model exchanges among agents impose critical privacy concerns. Despite some pioneering works to relieve such concerns, differentially private ADMM still confronts many research challenges. For example, the guarantee of differential privacy (DP) relies on the premise that the optimality of each local problem can be perfectly attained in each ADMM iteration, which may never happen in practice. The model trained by DP ADMM may have low prediction accuracy. In this paper, we address these concerns by proposing a novel (Improved) Plausible differentially Private ADMM algorithm, called PP-ADMM and IPP-ADMM. In PP-ADMM, each agent approximately solves a perturbed optimization problem that is formulated from its local private data in an iteration, and then perturbs the approximate solution with Gaussian noise to provide the DP  guarantee. To further improve the model accuracy and convergence, an improved version IPP-ADMM adopts sparse vector technique (SVT)  to determine if an agent should update its neighbors with the current perturbed solution. The agent calculates the difference of the current solution from that in the last iteration, and if the difference is larger than a threshold, it passes the solution to neighbors; or otherwise  the solution will be discarded. Moreover, we propose to track the total privacy loss under the zero-concentrated DP (zCDP) and provide a generalization performance analysis. Experiments on real-world datasets demonstrate that under the same privacy guarantee, the proposed algorithms are superior to the state of the art in terms of model accuracy and convergence rate.
 
\end{abstract}




\maketitle

\section{Introduction}
Nowadays, the development of machine learning creates many emerging applications that can improve the quality of our life, such as medical diagnosis, autonomous driving, face recognition, etc. With the proliferation of mobile phones and Internet-of-things devices, a vast amount of data have been generated at an ever-increasing rate, which leads to significant computational complexity for data collection and processes via a centralized machine learning approach. Therefore, distributed machine learning plays an increasingly important role in dealing with large scale machine learning tasks. There are many research efforts on distributed training a large scale optimization problem, which mainly consist of two types: (sub)gradient based methods, and Alternating Direction Method of Multipliers (ADMM) based methods. As shown in \cite{chang2014multi}, the convergence of (sub)gradient based methods are usually slow, which is $\mathcal{O}(1/{\sqrt{T}})$ under general convex objectives, while ADMM based algorithms can achieve $\mathcal{O}(1/{T})$ convergence rate, where $T$ is the number of iterations. Therefore, ADMM has been widely used in distributed machine learning \cite{zhang2014asynchronous,mota2013d}, and in this paper, we focus on the distributed ADMM.

In the framework of distributed ADMM, data providers (agents) collaboratively solve a learning problem, which can be decomposed into several subproblems, via an interactive procedure of local computation and message passing. However, the information exchanges during this process raise serious privacy concerns, and the adversary can extract private information from the shared learning models via various inference attacks \cite{shokri2017membership,melis2019exploiting}. To prevent such privacy leakage, differential privacy (DP) \cite{dwork2006calibrating} provides a de facto standard of privacy definition for protecting data privacy, which guarantees that the adversary with arbitrary background knowledge cannot extract any sensitive information about the training dataset. 

Many pioneering works have studied how to effectively integrate ADMM with DP, e.g., \cite{zhang2016dynamic,zhang2018improving,zhang2018recycled,huang2018dp,DingStochastic,ding2019optimal,dingbigdata}, which can be classified into two categories in general. The first type of works is to add a noisy term to perturb the objective function in each ADMM iteration using an objective perturbation approach \cite{chaudhuri2011differentially}. The second type of works is to perturb the updates of original distributed ADMM algorithm via an output perturbation approach \cite{chaudhuri2011differentially}. 
Specifically, as an objective perturbation method, \cite{zhang2016dynamic} proposed to inject noise to the update of the dual variable to provide DP guarantee, while the total privacy loss over the whole iterative process is not quantified. Further, \cite{zhang2018improving,zhang2018recycled} proposed to perturb the penalty parameter of ADMM and re-utilize the previous iteration's results to save the privacy loss. These methods also quantify the total privacy loss over the entire process. Moreover, \cite{huang2018dp} perturbed the augmented Lagrangian with time-varying Gaussian noise and considered a centralized network structure to perform ADMM updates.
As output perturbation methods, \cite{ding2019optimal,dingbigdata} proposed to perturb the primal variables by Gaussian noise with linearly decaying Gaussian noise to preserve DP and maintain the utility.

However, the guarantee of DP in the above works relies on the premise that the optimality of each local problem can be perfectly attained in each iteration during the whole training procedure, which is seldom seen in practice. Further, the trained models from the above works exhibit severe degradation in terms of the convergence performance and model accuracy, compared to their non-private versions.

In this paper, we propose (\textbf{I}mproved) \textbf{P}lausible differentially \textbf{P}rivate \textbf{ADMM} based distributed machine learning algorithm called PP-ADMM and IPP-ADMM, respectively. Instead of requiring each local problem to reach the optimality, PP-ADMM is able to release a noisy approximate solution of the local optimization with Gaussian noise related to the optimization accuracy, while preserving DP. 
  
To further improve the utility, we propose an improved version of PP-ADMM, i.e., IPP-ADMM, by exploiting the sparse vector technique (SVT) to check whether the current approximate solution has enough difference from that of the previous iteration. Moreover, the privacy analysis of our algorithms based on the zero-concentrated DP (zCDP) yields a tight privacy loss bound. We analyze the generalization performance of PP-ADMM. Our salient contributions are summarized as follows.

\begin{itemize}
 
\item To release the shackles of ``exact local optimal results" and make ADMM based distributed machine learning achieve DP and plausible, we propose a novel PP-ADMM method by exploiting the inexact solution of the perturbed local optimization over local agent's private data during each ADMM iteration, while preserving the data privacy.

\item We further propose an improved version of PP-ADMM (IPP-ADMM) by employing SVT to evaluate whether the current approximate solution has a big enough difference from that of the previous iteration. If the difference surpasses a predefined threshold, the approximate solution with Gaussian noise will be shared with neighbors; otherwise, the current approximate solution will be discarded. By this way, the redundant privacy loss accumulation and the transmissions of ``low quality'' can be avoided during the ADMM iterative process.

 \item To best track the privacy loss accumulation, we leverage the serial and parallel composition theorems of the zCDP to theoretically quantify and analyze the overall privacy guarantees of the PP-ADMM and IPP-ADMM algorithms. Moreover, we provide a generalization performance analysis of PP-ADMM by measuring the number of data samples required to achieve a certain criteria.
 
\item Through extensive experiments on the real-world datasets, we show the superior performance of the proposed algorithms over the state-of-the-art differentially private ADMM algorithms.
\end{itemize}

The rest of the paper is organized as follows.
Section~\ref{Problem formulation} formulates the optimization problem, and describes preliminaries of ADMM and differential privacy. Then, the plausible private robust ADMM algorithm and its corresponding privacy and sample complexity analysis are presented in Section~\ref{Private ADMM}. In Section~\ref{ippadmm}, we provide the IPP-ADMM, an improved version of PP-ADMM, and the corresponding privacy analysis. The experimental results on real-world datasets are shown in Section~\ref{experiments}. Finally, we draw conclusive remarks in Section~\ref{conclusion}.

\section{Problem Formulation and Preliminaries}\label{Problem formulation}
In this paper, we consider a connected network contains $N$ agents with node set $\mathscr{N}=\{1,\cdots,N\}$, and each agent $i$ has a dataset $D_i$ with $D_i=\{(x_i^n,y_i^n)\}_{n=1}^{|D_i|}$, where $x_i^n \in \mathcal{X}$ is a feature vector and $y_i^n \in \mathcal{Y}$ is a label. The communication among agents can be represented by an undirected graph $G=\{\mathscr{N},\mathscr{E}\}$, where $\mathscr{E}$ denotes the set of communication links between agents. Note that two agents $i$ and $j$ can communicate with each other only when they are neighbors, i.e., $(i,j)\in \mathscr{E} $.
We also denote the set of neighbors of agent $i$ as $\mathscr{B}_i$. The goal is to cooperatively train a classifier $\theta \in \mathbb{R}^d$ over the union of all local datasets in a decentralized fashion (i.e., no centralized controller) while keeping the privacy for each data sample, which can be formulated as an Empirical Risk Minimization (ERM) problem.
\begin{align}\label{erm}  
    \min_{\theta\in \mathbb{R}^d} \sum_{i=1}^N \frac{1}{|D_i|}\sum_{n = 1}^{|D_i|} \mathscr{L}( y_i^n \theta^T x_i^n) +\hat{\lambda}\mathscr{R}(\theta),
\end{align}
where $\mathscr{L}(\cdot):\mathcal{X} \times \mathcal{Y} \times \mathbb{R}^d \to \mathbb{R} $ stands for a convex loss function with $|\mathscr{L}'(\cdot)|\leq1$ and $0<\mathscr{L}''(\cdot)\leq c_1$, $\mathscr{R}(\theta): \mathbb{R}^d \to \mathbb{R}$ is a differentiable and 1-strongly convex regularizer to prevent overfitting, and $\hat{\lambda}\geq 0$ refers to a regularizer parameter that controls the impact of regularizer. We assume that each feature vector $x_i^n$ is normalized to $\|x_i^n\|_2\leq 1$.
Note that the formulations of classification in machine learning like logistic regression, or support vector machines, can also be fallen into the framework of ERM. In order to solve the ERM problem (\ref{erm}) in a decentralized manner, we adopt the simple but efficient optimization method, ADMM. We then in the following subsection review some preliminaries about ADMM algorithm for solving Problem (\ref{erm}).
\subsection{ADMM}
It is easy to see that the ERM problem (\ref{erm}) can be equivalently reformulated as the following consensus form by introducing $\theta_i$, that is, the local copy of common classifier $\theta$ at agent $i$.
\begin{equation} \label{eq:nc}
\begin{array}{cl}
\min\limits_{\{\theta_i\},\{\rho_{ij}\}} &\sum_{i=1}^N f_i(\theta_i) \\
\mbox{s.t.}\ &\theta_i = \rho_{ij},~ \theta_j = \rho_{ij}, ~i \in \mathscr{N},j\in \mathscr{B}_i, \\
\end{array} 
\end{equation} 
where $\{\rho_{ij} |i \in \mathscr{N},j\in \mathscr{B}_i\}$ is a set of slack variables to enforce all local copies are equal to each other, i.e., $\theta_1=\theta_2 = \cdots,=\theta_N$, and $f_i(\theta_i) = \frac{1}{|D_i|}\sum_{n = 1}^{|D_i|} \mathscr{L}( y_i^n \theta_i^T x_i^n ) + \frac{\hat{\lambda}}{N} \mathscr{R}(\theta_i)$. According to Problem (\ref{eq:nc}), each agent $i$ can minimize local function $f_i(\theta_i)$ over its own private dataset with respect to $\theta_i$, under the consensus constraints. In \cite{zhang2018improving}, ADMM is employed to optimize Problem (\ref{eq:nc}) in a decentralized fashion. By defining a dual variable $\lambda_i$ for agent $i$, and introducing the following notion,
$\mathcal{L}_{non}(\theta_i,D_i) =  f_i(\theta_i) + (2\lambda_i^t )^T\theta_i +  \eta \sum_{j \in \mathscr{B}_i}||\dfrac{1}{2}(\theta_i^{t}+\theta_j^{t})-\theta_i||_2^2$,
ADMM then has the following iterative updates in the $(t+1)$-th iteration:
 
 \begin{align}
    \theta_i^{t+1} = \underset{\theta_i}{\text{argmin}} ~~  \mathcal{L}_{non}(\theta_i,D_i);
    \label{eq:prelimi_10}
\end{align}
\begin{equation}\label{eq:prelimi_11}
\lambda_{i}^{t+1} = \lambda_{i}^{t} +  \dfrac{\eta}{2}\sum_{j \in \mathscr{B}_i}(\theta_i^{t+1}-\theta_j^{t+1}), 
\end{equation}
where $\eta>0$ is a penalty parameter. Note that the reason why the variable $\rho_{ij}$ is not appeared in (\ref{eq:prelimi_10}) and (\ref{eq:prelimi_11}) is that it can be expressed by using the primal variable $\theta_{i}$, as shown in \cite{mateos2010distributed}. In the iteration $t+1$, each agent $i \in \mathscr{N}$ updates its local $\theta_i^{t+1}$ via (\ref{eq:prelimi_10}) by using its previous results $\theta_i^{t}$ and $\lambda_i^{t}$, and the shared local classifiers $\theta_j^{t}$ from its neighbors $j \in \mathscr{B}_i$. Next, agent $i$ broadcasts $\theta_i^{t+1}$ to all its neighboring agents. After obtaining all of its neighboring computation results, each agent updates the dual variable $\lambda_i^{t+1}$ through (\ref{eq:prelimi_11}).

\subsection{Differential Privacy}
For the privacy-preserving data analysis, the standard privacy metric, Differential privacy (DP) \cite{dwork2006calibrating,dwork2014algorithmic}, is proposed to measure the privacy risk of each data sample in the dataset, and has already been adopted in many machine learning domains \cite{ding2020private, ding2020differentially, zhang2019differentially,shi2019deep,wang2018globe}. Basically, under DP framework, privacy protection is guaranteed by limiting the difference of the distribution of the output regardless of the value change of any one sample in the dataset. 
\begin{definition}[$(\epsilon,\delta)$-\bf{DP}~\cite{dwork2006calibrating}]
A randomized mechanism $\mathcal{M}$ satisfies $(\epsilon,\delta)$-DP if for any two neighboring datasets $D$ and $\hat{D}$ differing in at most one single data sample, and for any possible output $o \in Range(\mathcal{M})$, we have ${\rm Pr}[\mathcal{M}(D)= o] \leq e^{\epsilon} {\rm Pr}[\mathcal{M}(\hat{D}) = o] + \delta$.
\end{definition} 
Here $\epsilon,\delta$ are privacy loss parameters that indicate the strength of the privacy protection from the mechanism $\mathcal{M}$. The privacy protection is stronger if they are smaller. The above privacy definition reduces to pure DP when $\delta=0$ and when $\delta>0$, it is referred to as approximate DP. We can achieve pure and approximate DP by utilizing two popular approaches called Laplace and Gaussian Mechanism, both of which share the same form $\mathcal{M}(D) = \mathcal{M}_q (D) + {\rm\bf{u}}$, where $\mathcal{M}_q (D)$ is a query function over dataset $D$, and ${\rm\bf{u}}$ is random noise. We also denote two neighboring datasets $D$ and $\hat{D}$ as $D\sim\hat{D}$, and denote ${\rm Lap}(\lambda)$ as a zero-mean Laplace distribution with scale parameter $\lambda$.

\begin{definition}[\bf{Laplace Mechanism \cite{dwork2006calibrating}}]
Given any function $\mathcal{M}_q: \mathcal{D}\to \mathbb{R}^d$, the Laplace Mechanism is defined as: $\mathcal{M}_{L}(D,q,\epsilon) = \mathcal{M}_q (D) + {\rm\bf{u}}$, where $\rm\bf{u}$ is drawn from a Laplace distribution ${\rm Lap}(\frac{\Delta_1}{\epsilon})$ with scale parameter proportional to the $L_1$-sensitivity of $\mathcal{M}_q$ given as $\Delta_1 =\sup_{D\sim\hat{D}} \|\mathcal{M}_q(D)-\mathcal{M}_q(\hat{D})\|_1$. Laplace Mechanism preservers $\epsilon$-differential privacy.
\end{definition}
    
\begin{definition}[\bf{Gaussian Mechanism \cite{dwork2006calibrating}}]
Given any function $\mathcal{M}_q: \mathcal{D}\to \mathbb{R}^d$, the Gaussian Mechanism is defined as: $\mathcal{M}_{G}(D,q,\epsilon,\delta) = \mathcal{M}_q (D) + {\rm\bf{u}}$, where $\rm\bf{u}$ is drawn from a Gaussian distribution $\mathcal{N}(0,\sigma^2I_d)$ with $\sigma \geq \frac{\sqrt{2 \ln(1.25/\delta)  }\Delta_2}{\epsilon}$, and $\Delta_2$ is the $L_2$-sensitivity of function $\mathcal{M}_q$, i.e., $\Delta_2 = \sup_{D\sim\hat{D}}\|\mathcal{M}_q(D)-\mathcal{M}_q(\hat{D})\|_2$. Gaussian Mechanism provides $(\epsilon,\delta)$-differential privacy.
\end{definition}


Next, we introduce a generalization of DP, which is called the zero-concentrated DP (zCDP) \cite{bun2016concentrated} that uses the R\'enyi divergence between $\mathcal{M}(D)$ and $\mathcal{M}(\hat{D})$, which can achieve a much tighter privacy loss bound under multiple privacy mechanisms composition. 
\begin{definition}[\bf{$\rho$-zCDP \cite{bun2016concentrated}}]
We say that a randomized algorithm $\mathcal{M}$ provides $ \rho$-zCDP, if for all neighboring datasets $D$ and $\hat{D}$ and for all $\tau \in (1,\infty)$,  we have $ D_\tau (\mathcal{M}(D) \|\mathcal{M}(\hat{D})) \leq   \rho \tau,$
where $D_\tau (\mathcal{M}(D) \|\mathcal{M}(\hat{D}))$ is the $\tau$-R\'enyi divergence \footnote{Definition can be found in \cite{bun2016concentrated}} between the distribution $\mathcal{M}(D)$ and the distribution $\mathcal{M}(\hat{D})$. 
\end{definition}
The following lemmas show that the Gaussian mechanism satisfies zCDP, the composition theorem of $ \rho $-zCDP, and the relationship among $ \rho $-zCDP, $\epsilon$-DP, and $(\epsilon,\delta)$-DP.
\begin{lemma}[\cite{bun2016concentrated}]\label{Gaussian}
The Gaussian mechanism with noise $\mathcal{N}(0,\sigma^2)$ satisfies $ \Delta_2^2/ (2\sigma^2)$-zCDP.
\end{lemma}
\begin{lemma}[\bf{Serial Composition \cite{bun2016concentrated}}]\label{comp}
If randomized mechanisms $\mathcal{M}_1: \mathcal{D}\to \mathcal{R}_1$ and $ \mathcal{M}_2 : \mathcal{D}\to \mathcal{R}_2$ obey $ \rho_1 $-zCDP and $ \rho_2 $-zCDP, respectively. Then their composition defined as $\mathcal{M}'': D \to \mathcal{R}_1 \times \mathcal{R}_2$ by $ \mathcal{M}''=(\mathcal{M}_1,\mathcal{M}_2)$ obeys $ (\rho_1+\rho_2) $-zCDP.
\end{lemma}
\begin{lemma}[\bf{DP to zCDP conversion \cite{bun2016concentrated}}]\label{pure-tozcdp}
If a randomized mechanism $\mathcal{M}$ provides $\epsilon$-DP, then $\mathcal{M}$ is $ \frac{1}{2}\epsilon^2$-zCDP. Moreover, for $\mathcal{M}$ to satisfy $(\epsilon,\delta)$-DP, it suffices to satisfy $\rho$-zCDP with $\rho =\frac{\epsilon^2}{4\ln{(1/\delta)}}$. 
\end{lemma}
\begin{lemma}[\bf{zCDP to DP conversion \cite{bun2016concentrated}}]\label{cdptodp}
If a randomized mechanism $\mathcal{M}$ obeys $ \rho $-zCDP, then $\mathcal{M}$ obeys $ (\rho + 2\sqrt{\rho \ln(1/\delta)}, \delta)$-DP for all $0<\delta <1$.
\end{lemma}

\subsection{Sparse Vector Technique}
A powerful approach for achieving DP employs the sparse vector technique (SVT) \cite{dwork2009complexity}, which takes a sequence of queries with bounded sensitivity $\Delta$ against a fixed sensitive dataset and outputs a vector representing whether each answer to the query exceeds a threshold or not. A unique advantage of SVT is that it can output some query answer without paying additional privacy cost. Specifically, as shown in \cite{lyu2017understanding}, SVT has the following four steps. (i), We first compute a noisy threshold $\hat{\gamma}$ by adding a threshold noise ${\rm Lap}{(\frac{\Delta}{\epsilon_1})}$ to the predefined threshold $\gamma$. (ii), We then utilize a noise $v_i\sim {\rm Lap} (\frac{2c\Delta}{ \epsilon_2})$ to perturb each query $q_i$. (iii), We compare each noisy query answer $q_i(D)+\nu_i$ with the noisy threshold $\hat{\gamma}$ and respond whether it is higher ($\top$) or lower ($\perp$) than the threshold. (iv), This procedure continues until the number of $\top$'s in the output meets the predefined bound $c$. According to \cite{lyu2017understanding}, the SVT algorithm satisfies the $\epsilon$-DP with $\epsilon=\epsilon_1+\epsilon_2$. In order to analyze the privacy guarantee of SVT under the zCDP framework, we utilize the conversion result in Lemma \ref{pure-tozcdp}. We can see that SVT satisfies $\frac{1}{2}\epsilon^2$-zCDP.
\section{Plausible Private ADMM}\label{Private ADMM}
In this section, we will present our plausible differentially private (PP-ADMM) by adding Gaussian noise related to the maximum tolerable gradient norm of perturbed objective in each ADMM iteration, which relaxes the requirement of exact optimal solution as shown in \cite{zhang2016dynamic,zhang2018improving,zhang2018recycled}, to provide differential privacy guarantee of each training data sample during the iterative procedure. We also adopt the privacy framework of zCDP to compute much tighter privacy loss estimation of PP-ADMM. In addition, the generalization performance guarantees of PP-ADMM is provided by measuring the number of data samples at each agent to achieve a specific criteria.

Specifically, in each iteration, we perturb the subproblem (\ref{eq:prelimi_10}) with the objective perturbation method the same as used in previous studies \cite{zhang2016dynamic,zhang2018improving,zhang2018recycled}, where a random linear vector $(b_{i1})^{T} \theta_i$ is injected to the objective function, and $b_{i1}$ is a random vector drawn from a zero mean Gaussian distribution $\mathcal{N}(0, {\sigma}_{i1}^2 I_{d})$. Consequently the objective function (\ref{eq:prelimi_10}) used to update the primal variable $\theta_i^{t+1}$ becomes the following modified function: 
 
\begin{align}
\mathcal{L}_{per}(\theta_i,D_i) =  f_i(\theta_i) + (2\lambda_i^t + b_{i1})^T\theta_i +  \eta \sum_{j \in \mathscr{B}_i}||\dfrac{1}{2}(\theta_i^{t}+\theta_j^{t})-\theta_i||_2^2 \label{eq:perturbed_loss}
\end{align}
where $f_i(\theta_i) = \frac{1}{|D_i|}\sum_{n = 1}^{|D_i|} \mathscr{L}( y_i^n \theta_i^T x_i^n ) + \frac{\hat{\lambda}}{N} \mathscr{R}(\theta_i)$. 
In order to ensure DP guarantee, as pointed out in \cite{zhang2016dynamic,zhang2018improving,zhang2018recycled}, 
each agent $i \in \mathscr{N}$ needs to find the optimal solution $\Tilde{\theta}_{i}^{t+1}$ of the perturbed objective function $\mathcal{L}_{per}(\theta_i,D_i)$, i.e., 
\begin{align}
    \Tilde{\theta}_{i}^{t+1} = \underset{\theta_i}{\text{argmin}} ~~\mathcal{L}_{per}(\theta_i,D_i).  \label{perturbedloss}
\end{align}
However, the subproblem (\ref{perturbedloss}) may not be easy to solve and obtain an optimal solution in a finite time. For instance, if we choose logistic regression as loss function, the subproblem (\ref{perturbedloss}) cannot yield an analytical solution due to the complicated form of logistic regression. Especially when the problem dimension or the number of training samples is large, obtaining optimal solution might not be feasible in every iteration.

Thus, we consider obtaining the approximate solution of perturbed objective function $\mathcal{L}_{per}(\theta_i,D_i)$ to provide privacy guarantees when the optimal solution is not attainable. 
  
Specifically, we approximately solve the perturbed problem until the norm of gradient of $\mathcal{L}_{per}$ is within a pre-defined threshold $\beta$. However, due to the limitations of objective perturbation method \cite{chaudhuri2011differentially}, releasing this inexact solution leads to the failure of providing DP guarantee. We thus perturb the approximated solution $\hat{\theta}_i^{t+1}$ with another random noise $b_{i2}$ from Gaussian distribution $\mathcal{N}(0, {\sigma}_{i2}^2 I_{d})$, to "fuzz" the difference between $\hat{\theta}_{i}^{t+1}$ and the optimal solution $\Tilde{\theta}_{i}^{t+1}$. Note that the noise variance ${\sigma}_{i2}^2$ has the parameter $\beta$ about the maximum tolerable gradient norm, which leads to a trade-off between the gradient norm bound and the difficulty of obtaining an approximate solution within the norm bound.

\begin{algorithm}[!htb]
\caption{Plausible Private ADMM}
\label{alg:PADMM}
\algsetup{indent=2em}
\begin{algorithmic}[1]
\STATE \textbf{Input:} datasets $\{D_i\}_{i=1}^N$; initial variables $\theta_i^0 \in \mathbb{R}^d$ and $\lambda_i^0 = 0_{d}$; step size $\eta$; privacy parameters, $\epsilon_{i1}$, $\delta_{i1}$, $\epsilon_{i3}$, $\rho_{i2}$; { Optimizer} $\mathscr{O}(\cdot,\cdot):\mathscr{F} \times {\bm{\beta}} \to \mathbb{R}^d$ ( $\mathscr{F}$ is the class of objectives, and $\bm{\beta}$ is the optimization accuracy, i.e., the gradient norm of objectives); gradient norm threshold $\beta \in {\bm \beta}$.
\STATE Set $\epsilon_{i1}$, $\delta_{i1}$, $\epsilon_{i3}$, $\rho_{i2}>0$ such that $\epsilon_{i1}>\epsilon_{i3}$.
\STATE Set regularizer parameter $\hat{\lambda}\geq\underset{i}\max \frac{2.8N c_1}{(\epsilon_{i1}-\epsilon_{i3})|D_i|}$.
\FOR {$t =0,\dots, T-1$}
\FOR{$i=1,\dots,N$}
\STATE Generate noise $b_{i1}\sim \mathcal{N}(0, {\sigma}_{i1}^2 I_{d})$ with $\sigma_{i1} = {2\sqrt{2\ln{(1.25/\delta_{i1})}}}/(|D_i|\epsilon_{i3})$.
\STATE Construct the perturbed objective function  $\mathcal{L}_{per}(\theta_i,D_i)$ according to (\ref{eq:perturbed_loss}). 
\STATE Compute an approximate solution $\hat{\theta}_i^{t+1}$: $\hat{\theta}_{i}^{t+1} = \mathscr{O}(\mathcal{L}_{per}(\theta_i,D_i),\beta)$.
\STATE Generate noise $b_{i2}\sim \mathcal{N}(0, {\sigma}_{i2}^2 I_{d})$ with ${\sigma}_{i2} = {\beta}/[{\sqrt{2\rho_{i2}}(\frac{\hat{\lambda}}{N}+2\eta|\mathscr{B}_i|)}]$. 
\STATE Perturb $\hat{\theta}_i^{t+1}$: $\theta_i^{t+1} = \hat{\theta}_i^{t+1} +b_{i2}$.
\ENDFOR
\FOR{$i=1,\dots,N$}
\STATE Broadcast $\theta_i^{t+1}$ to all neighbors $j\in \mathscr{B}_i$.
\ENDFOR
\FOR{$i=1,\dots,N$}
\STATE Update local dual variables $\lambda_{i}^{t+1}$ from $\lambda_{i}^{t+1} = \lambda_{i}^{t} +  \dfrac{\eta}{2}\sum\limits_{j \in \mathscr{B}_i}(\theta_i^{t+1}-\theta_j^{t+1})$.
\ENDFOR
\ENDFOR
\end{algorithmic}
\end{algorithm}

The key steps of PP-ADMM algorithm are summarized in Algorithm \ref{alg:PADMM}. The privacy parameters $(\epsilon_{i1},\delta_{i1})$ are used to perturb the objective function while the parameter $\rho_{i2}$ being used to perturb the approximate solution. Moreover, the parameter $\epsilon_{i3}$, a portion of $\epsilon_{i1}$, is used to scale the noise injected to the objective function, and the remaining privacy budget $(\epsilon_{i1}-\epsilon_{i3})$ is allocated to setting the regularizer parameter.
Notice that we also define an {\bf Optimizer} $\mathscr{O}(\cdot,\cdot):\mathscr{F} \times {\bm{\beta}} \to \mathbb{R}^d$, where $\mathscr{F}$ is the class of objectives, and $\bm{\beta}$ is the optimization accuracy, i.e., the gradient norm of objectives. Each agent $i$ then constructs the perturbed function $\mathcal{L}_{per}(\theta_i,D_i)$ with a Gaussian random vector $b_{i1}$ and finds an inexact solution $\hat{\theta}_i^{t+1}$ where the norm of gradient is lower than $\beta$, i.e.,  $\hat{\theta}_{i}^{t+1} = \mathscr{O}(\mathcal{L}_{per}(\theta_i,D_i),\beta)$.
After that each agent $i$ generates a random Gaussian noise $b_{i2}$ and transmits $\theta_i^{t+1} = \hat{\theta}_i +b_{i2} $ to its neighbors $j \in \mathscr{B}_i$. Finally, each agent updates the local dual variables $\lambda_{i}^{t+1}$ via $\lambda_{i}^{t+1} = \lambda_{i}^{t} +  \dfrac{\eta}{2}\sum\limits_{j \in \mathscr{B}_i}(\theta_i^{t+1}-\theta_j^{t+1})$.

\subsection{Privacy Analysis}
Here, we provide the privacy guarantee of PP-ADMM (Algorithm \ref{alg:PADMM}) in the following two theorems. Due to the limited space, we only provide a proof idea of Theorem \ref{PP}, and the detailed proof can be found in Appendix.

\begin{theorem}
The PP-ADMM in Algorithm \ref{alg:PADMM} satisfies $\rho_i$-zCDP for each agent $i$ with $\rho_i = T(\rho_{i1}+\rho_{i2})$, where $\rho_{i1} = \epsilon_{i1}^2/(4\ln{(1/\delta_{i1})})$, and $\rho_{i2}>0$ is the privacy budget for perturbing the approximate solution.
\end{theorem}\label{PP}
\begin{proof-idea}
For achieving $\rho_i$-zCDP for each agent $i$ at $t+1$ iteration in Algorithm \ref{alg:PADMM}, we first divide the output of $t+1$ iteration into two parts. The first part is to obtain the optimal solution $\Tilde{\theta}_{i}^{t+1}$ of the perturbed objective function $\mathcal{L}_{per}(\theta_i,D_i)$, and the second part is to obtain the approximate solution with Gaussian noise ${\theta}_{i}^{t+1}$. We then show obtaining the optimal solution $\Tilde{\theta}_{i}^{t+1}$ provides $\rho_{i1}$-zCDP with $\rho_{i1} =\epsilon_{i1}^2/(4\ln{(1/\delta_{i1})})$ for the first part, and releasing an approximate solution in the second part is $\rho_{i2}$-zCDP. By using the composition of zCDP in Lemma \ref{comp}, we can get releasing the perturbed primal variable $\theta_{i}^{t+1}$ at $t+1$ iteration provides $(\rho_{i1}+\rho_{i2})$-zCDP. Considering $T$ iterations, the total privacy loss for each agent $i$ is bounded by $\rho_i = T(\rho_{i1}+\rho_{i2})$.
\end{proof-idea} 
We then give the following parallel composition theorem of $\rho$-zCDP to provides a together characterization of total privacy loss for distributed algorithms. 
\begin{lemma}[{Parallel Composition} \cite{8835283}]\label{parallel}
Suppose that a mechanism $\mathcal{M}$ consists of a sequence of $k$ adaptive mechanism $\mathcal{M}_1,\cdots,\mathcal{M}_k$ where each $\mathcal{M}_i:\prod_{j=1}^{i-1}\mathcal{R}_j\times\mathcal{D}\to\mathcal{R}_i$ and $\mathcal{M}_i$ satisfies $\rho_i$-zCDP. Let $\mathbb{D}_1,\mathbb{D}_2,\cdots,\mathbb{D}_k$ be the result of a randomized partition of the input domain $\mathbb{D}$. The mechanism $\mathcal{M}(D) = (\mathcal{M}_1(D\cap \mathbb{D}_1),\cdots,\mathbb{M}_{k}(D\cap\mathbb{D}_k))$ satisfies $(\underset{i}{\max}~\rho_i)$-zCDP.
\end{lemma}

Based on Lemma \ref{parallel}, we can directly obtain the total privacy loss of PP-ADMM given as follows.
\begin{theorem}
The PP-ADMM in Algorithm \ref{alg:PADMM} satisfies $\rho$-zCDP with $\rho=\underset{i}{\max}~\rho_i$ and satisfies $(\epsilon,\delta)$-DP with $\epsilon = \rho+2\sqrt{\rho\ln{(1/\delta)}}$.
\end{theorem}

\subsection{Sample Complexity Analysis}
We next measure the generalization performance of PP-ADMM by focusing on the ERM problem given in Section \ref{Problem formulation}. We also assume that data samples of each agent $i$ are drawn from a data distribution $\mathscr{P}$. The expected loss of classifier $\theta_i^{t}$ at iteration $t$ is defined as 
\begin{align*}
    \mathbb{L}(\theta_i^{t}) = \mathbb{E}_{(x,y)\sim \mathscr{P}}\left(\mathscr{L}( y (\theta_i^{t})^T x)\right).
\end{align*}
Following the similar analysis in  \cite{chaudhuri2011differentially,zhang2016dynamic}, we first introduce a reference classifier $\theta_{ref}$ with expected loss $\mathbb{L}(\theta_{ref})$, and we then measure the generalization performance using the number of samples $D_i$ required at each agent to achieve $\mathbb{L}(\theta_i^t)\leq \mathbb{L}(\theta_{ref})+ a_{acc}$, where $a_{acc}$ is the generalization error.

\subsubsection{PP-ADMM without Noise}

Here, we consider the learning performance at all iterations rather than only the final output. 
Let the intermediate updated classifier $\hat{\theta}_{i,non}^{t+1}$ at iteration $t+1$ be 
 $\hat{\theta}_{i,non}^{t+1} = \mathscr{O}(\mathcal{L}_{non}(\theta_i,D_i),\beta)$.
Note that $\{\hat{\theta}_{i,non}^{t+1}\}$ is a sequence of non-private classifier without adding perturbations. 
Let $\theta^* = \underset{\theta_i}{\text{argmin}}~~ f_i{(\theta_i,D_i)}$ be the optimal output of PP-ADMM without Noise.
The sequence $\{\hat{\theta}_{i,non}^{t+1}\}$ is bounded and $\hat{\theta}_{i,non}^{t+1}$ converges to $\theta^*$ as $t\to \infty$.
Therefore, there exists a constant $\Delta_{i,non}^{t+1}$ at iteration $t+1$ such that $\mathbb{L}(\hat{\theta}_{i,non}^{t+1})\leq \mathbb{L}(\theta^*)+ \Delta_{i,non}^{t+1}$. We then have the following result, and the detailed proof can be found in supplemental material.

\begin{theorem}\label{NPP-ADMM}
Consider a regularized ERM problem with $\mathscr{R}(\theta) = \frac{1}{2}\|\theta\|_2^2$, and let $\theta_{ref}$ be the reference classifier for all agents and $\{\hat{\theta}_{i,non}^{t+1}\}$ be a sequence of outputs of PP-ADMM without adding noise.
If the number of samples at agent $i$ satisfies,
\begin{align*}
    |D_i| \geq V \max_{t}\{\frac{ \log{(1/\xi)}}{\frac{ (a_{acc}-\Delta_{i,non}^{t+1})^2}{2\|\theta_{ref}\|^2_2}-(1+a)\beta}\}
\end{align*}
for some constant $V$, then $\hat{\theta}_{i,non}^{t+1}$ satisfies 
\begin{align*}
    \Pr\left[\mathbb{L}(\hat{\theta}_{i,non}^{t+1})\leq \mathbb{L}(\theta_{ref})+ a_{acc} \right]\geq 1-\xi
\end{align*}
with $a_{acc}\geq\Delta_{i,non}^{t+1}$.
\end{theorem}

\begin{remark}

As we can see from Theorem \ref{NPP-ADMM}, the number of data samples $|D_i|$ relies on the $l_2$-norm of reference classifier $\|\theta_{ref}\|_2^2$ and the parameter $\beta$ that bounds the optimization accuracy of the non-private intermediate classifier. The results demonstrate that if $|D_i|$ satisfies 
$ |D_i| \geq V \max_{t}\{\frac{ \log{(1/\xi)}}{\frac{ (a_{acc}-\Delta_{i,non}^{t+1})^2}{2\|\theta_{ref}\|^2_2}-(1+a)\beta}\}$, each agent's non-private intermediate classifier will have an additional error less than $a_{acc}$ compared to any classifier with $\|\theta_{ref}\|_2^2$. Moreover, if $\beta = 0$, the result reduces to $ |D_i| \geq V \max_{t}\{\frac{ 2\|\theta_{ref}\|^2_2\log{(1/\xi)}}{ (a_{acc}-\Delta_{i,non}^{t+1})^2}\}$, the same as given in \cite{zhang2016dynamic}, which shows that the lower optimization accuracy of the non-private intermediate classifier, the more samples required to achieve the same accuracy.
\end{remark}

\subsubsection{PP-ADMM}

We then show the sample complexity of the PP-ADMM algorithm. Similar to the analysis in PP-ADMM without noise, we also consider bounding the generalization error of the intermediate classifier $\theta_i^{t+1}$ of each agent $i$ at all iterations. In order to compare the private classifier $\theta_i^{t+1}$ with a reference classifier $\theta_{ref}$, we follow the same strategy used in \cite{zhang2016dynamic}. We define a new optimization function $f_i^{new}(\theta_i,D_i)=f_i(\theta_i,D_i)+  {b_{i1}}^T\theta_i$ and 
then solving PP-ADMM algorithm is equivalent to solving a new optimization problem, where each agent $i$'s performs local minimization to get $\theta_i^{t+1} = \mathscr{O}(f_i^{new}(\theta_i,D_i),\beta) + b_{i2}$.
The sequence of outputs $\{\theta_i^{t+1}\}$ is bounded and $\theta_i^{t+1}$ converges to a fixed point $\theta^*_{new}$ as $t \to \infty$. Thus, there exists a constant $\Delta_{i,new}^{t+1}$ at $t+1$ iteration, such that $\mathbb{L}({\theta}_{i}^{t+1})\leq \mathbb{L}(\theta^*_{new})+ \Delta_{i,new}^{t+1}$. We then give the following result, and the detailed proof can be found in supplemental material.

\begin{theorem}\label{PPADMM_perfromacne}
Consider a regularized ERM problem with $\mathscr{R}(\theta) = \frac{1}{2}\|\theta\|_2^2$, and let $\theta_{ref}$ be the reference classifier for all agents and $\{{\theta}_{i}^{t+1}\}$ be a sequence of outputs of PP-ADMM.
If the number of samples at agent $i$ satisfies, for some constant $V$,
\begin{align*}
    |D_i| \geq V \max_{t}\{\frac{ \log{(1/\xi)}}{\frac{ (a_{acc}-\Delta_{i,new}^{t+1})^2}{2\|\theta_{ref}\|^2_2}-(1+a)(\beta + \mathscr{H} )}\}
\end{align*}
with $\mathscr{H} = \frac{\sigma_{i2}(a_{acc}-\Delta_{i,new}^{t+1})\sqrt{2d\log\frac{1}{\xi}}}{\|\theta_{ref}\|^2_2} +2\sigma_{i1}^2d\log{\frac{1}{\xi}} $, then $\hat{\theta}_{i,new}^{t+1}$ satisfies 
\begin{align*}
    \Pr\left[\mathbb{L}({\theta}_{i}^{t+1})\leq \mathbb{L}(\theta_{ref})+ a_{acc} \right]\geq 1-3\xi
\end{align*}
with $a_{acc}\geq\Delta_{i,new}^{t+1}$.
\end{theorem}

\begin{remark}
Compared to Theorem \ref{NPP-ADMM}, we can see that in Theorem \ref{PPADMM_perfromacne}, the privacy constraints impose an additional term $\mathscr{H}$ with $\mathscr{H}\\= {\sigma_{i2}(a_{acc}-\Delta_{i,new}^{t+1})\sqrt{2d\log\frac{1}{\xi}}}/\|\theta_{ref}\|^2_2 +2\sigma_{i1}^2d\log{\frac{1}{\xi}} $. If both noise variances $\sigma_{i1}$ and $\sigma_{i2} $ are equal to zero, the number of required samples $|D_i|$ will reduce to the same result shown in Theorem \ref{NPP-ADMM}. Moreover, the additional term $\mathscr{H}$ demonstrates that the higher dimension of features, the more added noise to achieve the same accuracy requires more data samples.
\end{remark}

\section{Improved Plausible Private ADMM}\label{ippadmm}
In this section, we present an improved version of PP-ADMM algorithm called Improved Plausible Private ADMM (IPP-ADMM) by leveraging sparse vector technique (SVT) to improve the performance and reduce the communication cost of PP-ADMM. Compared with current differentially private ADMM algorithms \cite{zhang2016dynamic,zhang2018improving,zhang2018recycled}, although the proposed PP-ADMM algorithm can ensure DP guarantee without requiring the optimal solution during each ADMM iteration, the primal variable is updated using the local data in every iteration and frequently broadcasted to neighboring agents, which leads to the privacy loss unavoidably accumulating over the iterations, and compromise the accuracy during the whole training procedure. 

Hence, we adopt SVT that can output some local computational results without paying any privacy budget, to check whether current approximate solution has a big enough difference from that of previous iteration,
 
where the difference is quantified by a quality function, $f_i(\theta_i^t)-f_i(\hat{\theta}_i^{t+1})$, based on the change of the values of local function over the primal variable from previous iteration and current approximate solution. If a sufficient level of difference $\alpha$ is achieved, each agent transmits the current approximate solution with Gaussian noise to its neighbors. Intuitively, if the difference between the current approximate solution $\hat{\theta}_i^{t+1}$ and the previously transmitted $\theta_i^t$ is small, then using either one does not help the convergence of the iterative process, which leads to reducing the communication cost.

However, one difficulty in using SVT is that there is no known priori bound on query (i.e., the quality function) $f_i(\theta_i^t)-f_i(\hat{\theta}_i)$. To bound the sensitivity of $f_i(\theta_i^t)-f_i(\hat{\theta}_i)$, we apply the clipping method to clipping the loss function  $\mathscr{L}(\cdot)$. Given a fixed clipping threshold $C_{loss}$, we compute the value of loss function $\mathscr{L}(\cdot)$ on each local data sample, clip the values at most $C_{loss}$, and compute the value of $f_i(\theta_i^t)-f_i(\hat{\theta}_i)$ based on the clipped values. Note that we denote this loss function clipping procedure as {\texttt{Clip}}. 
\begin{algorithm}[!htb]
\caption{Improved Plausible Private ADMM Run by Agent $i$}
\label{alg:IPPADMM}
\algsetup{indent=2em}
\begin{algorithmic}[1]
\STATE \textbf{Input:} dataset $D_i$; initial variables $\theta_i^0 \in \mathbb{R}^d$ and $\lambda_i^0 = 0_{d}$; threshold, $\alpha$; Maximum number of primal variables that can be broadcasted, $\mathbf{c}$; loss function clipping threshold $C_{loss}$; step size $\eta$; privacy parameters, $\epsilon_{i1}$, $\delta_{i1}$, $\epsilon_{i3}$, $\rho_{i2}$, $\epsilon_1$, $\epsilon_2$; { Optimizer} $\mathscr{O}(\cdot,\cdot):\mathscr{F} \times {\bm{\beta}} \to \mathbb{R}^d$ ( $\mathscr{F}$ is the class of objectives, and $\bm{\beta}$ is the optimization accuracy, i.e., the gradient norm of objectives); gradient norm threshold $\beta \in {\bm \beta}$.
\STATE Set $\epsilon_{i1}$, $\delta_{i1}$, $\epsilon_{i3}$, $\rho_{i2}$, $\epsilon_1$, $\epsilon_2 >0$ such that $\epsilon_{i1}>\epsilon_{i3}$.
\STATE Set regularizer parameter $\hat{\lambda}\geq \underset{i}\max \frac{2.8N c_1}{(\epsilon_{i1}-\epsilon_{i3})|D_i|}$.
\STATE $count_i = 0$.
\FOR {$t =0,\dots, T-1$}
\STATE Generate noise $b_{i1}\sim \mathcal{N}(0, {\sigma}_{i1}^2 I_{d})$ with $\sigma_{i1} = {2\sqrt{2\ln{(1.25/\delta_{i1})}}}/(|D_i|\epsilon_{i1})$.
\STATE Construct the perturbed objective function  $\mathcal{L}_{per}(\theta_i,D_i)$ according to (\ref{eq:perturbed_loss}). 
\STATE Compute an approximate solution $\hat{\theta}_i^{t+1}$: $\hat{\theta}_{i}^{t+1} = \mathscr{O}(\mathcal{L}_{per}(\theta_i,D_i),\beta)$.
\IF {$
    \displaystyle \texttt{Clip}\left[f_i(\theta_i^t)-f_i(\hat{\theta}_i^{t+1})\right] + {\rm Lap}(\frac{4 {\mathbf{c}}{C}_{loss}}{\epsilon_2})\geq \alpha +{\rm Lap}(\frac{2 {\rm \mathbf{c}}{C}_{loss}}{\epsilon_1})$}
\STATE $count_i = count_i +1$, {\bf{Abort}} if $count_i> {\rm \mathbf{c}}$.
\STATE Generate noise $b_{i2}\sim \mathcal{N}(0, {\sigma}_{i2}^2 I_{d})$ with ${\sigma}_{i2} = {\beta}/[{\sqrt{2\rho_{i2}}(\frac{\hat{\lambda}}{N}+2\eta|\mathscr{B}_i|)}]$. 
\STATE Perturb $\hat{\theta}_i^{t+1}$: $\theta_i^{t+1} = \hat{\theta}_i^{t+1} +b_{i2}$.
\STATE Broadcast $\theta_i^{t+1}$ to all neighbors $j\in \mathscr{B}_i$.
\ELSE 
\STATE Let $\theta_i^{t+1} = \theta_i^t$.
\ENDIF
\IF   {$\theta_j^{t+1}$ is not received from neighbor $j\in\mathscr{B}_i$}
\STATE Replace $\theta_j^{t+1} $ with $\theta_j^{t}$.
\ELSE 
\STATE Keep $\theta_j^{t+1}$.
\ENDIF
\STATE Update local dual variables $\lambda_{i}^{t+1}$ from $\lambda_{i}^{t+1} = \lambda_{i}^{t} +  \dfrac{\eta}{2}\sum\limits_{j \in \mathscr{B}_i}(\theta_i^{t+1}-\theta_j^{t+1})$.
\ENDFOR
\end{algorithmic}
\end{algorithm}

The complete procedure of IPP-ADMM algorithm for a single agent is shown in Algorithm \ref{alg:IPPADMM}. The privacy parameters $\epsilon_1$ and $\epsilon_2$ are allocated to perturb the quality function and threshold $\alpha$, respectively. In each iteration, each agent $i$ first constructs the perturbed function $\mathcal{L}_{per}(\theta_i,D_i)$ with a Gaussian random vector $b_{i1}$ and finds an inexact solution $\hat{\theta}_i^{t+1}$, where the norm of gradient is lower than $\beta$, i.e., $\hat{\theta}_{i}^{t+1} = \mathscr{O}(\mathcal{L}_{per}(\theta_i,D_i),\beta)$. Then each agent apply the clipping method \texttt{Clip} to clip the quality function $f_i(\theta_i^t)-f_i(\hat{\theta}_i^{t+1})$ with a clipping threshold $C_{loss}$ to limit the sensitivity of quality function. Further, each agent uses SVT to check whether the difference between the approximate solution $\hat{\theta}_i^{t+1}$ and $\theta_i^t$ is below a noisy threshold $\hat{\alpha}=\alpha +{\rm Lap}(\frac{2 {\rm \mathbf{c}}{C}_{loss}}{\epsilon_1})$ via a noisy quality function, $\texttt{Clip} \left[f_i(\theta_i^t)-f_i(\hat{\theta}_i^{t+1})\right] + {\rm Lap}(\frac{4 {\mathbf{c}}{C}_{loss}}{\epsilon_2})$. If yes, then agent $i$ does not transmit any computational results and let $\theta_i^{t+1} = \theta_i^t$; otherwise, each agent $i$ generates a random noise $b_{i2}\sim \mathcal{N}(0, {\sigma}_{i2}^2 I_{d})$ with ${\sigma}_{i2} = {\beta}/{\sqrt{2\rho_{i2}}(\frac{\hat{\lambda}}{N}+2\eta|\mathscr{B}_i|)}$, and transmits $\theta_i^{t+1} = \hat{\theta}_i^{t+1} +b_{i2} $ to its neighbors. Moreover, each agent maintains a counter ${count}_i$ to bound the total number of broadcasts of primal variables during the whole interactive process. If a predefined transmission number ${\rm \mathbf{c}} ({\rm \mathbf{c}} \leq T)$ is exceeded, agent $i$ stops transmitting anything even when the condition in Line 7 is satisfied. Hence, if agent $i$ does not receive $\theta_j^{t+1}$ from any neighbor $j \in \mathscr{B}_i$, then lets $\theta_j^{t+1} =\theta_j^{t}$. Finally, each agent updates the local dual variables $\lambda_{i}^{t+1}$ via $\lambda_{i}^{t+1} = \lambda_{i}^{t} +  \dfrac{\eta}{2}\sum\limits_{j \in \mathscr{B}_i}(\theta_i^{t+1}-\theta_j^{t+1})$.

\subsection{Privacy Analysis}
We provide the privacy guarantee of IPP-ADMM (Algorithm \ref{alg:IPPADMM}) in following theorem.

\begin{theorem}\label{ipp}
The IPP-ADMM in Algorithm \ref{alg:IPPADMM} satisfies $\rho'_i$-zCDP for each agent $i$ with $\rho'_i = \rho'_1+\mathbf{c}(\rho_{i1}+\rho_{i2})$, where $\rho'_1 = \frac{(\epsilon_1+\epsilon_2)^2}{2}$, $ \rho_{i1}=\epsilon_{i1}^2/(4\ln{(1/\delta_{i1})}) $,  $\rho_{i2}>0$ is the privacy budget for perturbing the approximate solution, and $\mathbf{c}$ $(\mathbf{c}<T)$ is the maximum number of primal variables that can be broadcasted. Moreover, the total privacy guarantee of IPP-ADMM is $\rho'$-zCDP with $\rho' = \underset{i}{\max} ~\rho'_i$.
\end{theorem} \vspace{-8.5pt}
\begin{proof}
For achieving $\rho'_i$-zCDP for each agent $i$ in Algorithm \ref{alg:IPPADMM}, we first divide the procedure of the algorithm into two parts. The first part is using SVT to compare the noisy threshold and the perturbed query answer (i.e., the value of quality function) to check the quality of the approximate solution obtained in Step 7 of the Algorithm \ref{alg:PADMM}. The second part is to share the approximate solution with Gaussian noise, whose value is above the threshold. We prove that DP mechanism used in the first part provides $\rho'_1$-zCDP (shown in Lemma \ref{svt1}). Moreover, at each iteration, the privacy budget spending on releasing an approximate solution in the second part is $(\rho_{i1}+\rho_{i2})$-zCDP (shown in Theorem \ref{PP}).
Then, using the composition of zCDP in Lemma \ref{comp}, we obtain the privacy guarantee of IPP-ADMM for each agent $i$ is $\rho_i = \rho_1+\mathbf{c}(\rho_{i1}+\rho_{i2})$ by considering $\mathbf{c}$ times of broadcasting primal variables.
Lastly, we get a total privacy guarantee of IPP-ADMM, i.e., $\rho'$-zCDP with $\rho' = \underset{i}{\max} ~\rho'_i$ by adopting the parallel composition in Lemma \ref{parallel}.
\end{proof} 
Before presenting the privacy guarantee of the first part, i.e., compare the noisy threshold and the perturbed query answer to check the quality of the approximate solution, we first give the sensitivity of the clipped quality function as follows.
\begin{lemma}\label{sens_qua}
Given a clipping threshold $C_{loss}$ of the loss function $\mathscr{L}(\cdot)$, the sensitivity of quality function $f_i(\theta_i^t)-f_i(\hat{\theta}_i^{t+1})$ is at most $2C_{loss}$, where $f_i(\theta_i) = \frac{1}{|D_i|}\sum_{n = 1}^{|D_i|} \mathscr{L}( y_i^n \theta_i^T x_i^n ) + \frac{\hat{\lambda}}{N} \mathscr{R}(\theta_i)$.
\end{lemma}
 
\begin{proof}
Fix a pair of adjacent datasets $D_i$ and $\hat{D}_i$ and we also assume that only the first data point in $D_i$ and $\hat{D}_i$ are different, i.e., $(x_i^1, y_i^1)$ and $(\hat{x}_i^1,\hat{y}_i^1)$. According to the definition of $L_1$-sensitivity, we have
\begin{align*}
    \Delta_f =~~~& \|f_i(\theta_i^t,D_i)-f_i(\hat{\theta}_i^{t+1},D_i)-f_i(\theta_i^t,\hat{D}_i)+f_i(\hat{\theta}_i^{t+1},\hat{D}_i)\|_1\\
    =~~~&\|\mathscr{L}( y_i^1 (\theta_i^t)^T x_i^1 )-\mathscr{L}( \hat{y}_i^1 (\theta_i^t)^T \hat{x}_i^1 ) \\
   &-(\mathscr{L}( y_i^1 (\theta_i^{t+1})^T x_i^1 )-\mathscr{L}( \hat{y}_i^1 (\theta_i^{t+1})^T \hat{x}_i^1 ) )\|_1\\
    \leq ~~~&\|\mathscr{L}( y_i^1 (\theta_i^t)^T x_i^1 )-\mathscr{L}( \hat{y}_i^1 (\theta_i^t)^T \hat{x}_i^1 )\|_1\\
 &  +\|\mathscr{L}( y_i^1 (\theta_i^{t+1})^T x_i^1 )-\mathscr{L}( \hat{y}_i^1 (\theta_i^{t+1})^T \hat{x}_i^1 ) \|_1\\
    \leq~~~& 2C_{loss}.\tag*{\qedhere}
\end{align*}
\end{proof}
 
Then we show the privacy guarantee of the first part in the following lemma.
\begin{lemma}\label{svt1}
Given the maximum number of primal variables that we can broadcast, $\mathbf{c}$, using SVT to check whether the approximate solution is above the threshold $\alpha$ provides $\rho_1$-zCDP with $\rho_1 = \frac{(\epsilon_1+\epsilon_2)^2}{2}$.
\end{lemma}
\begin{proof}
During the whole training process, each agent will receive a stream of queries (i.e., a stream of clipped quality functions $\texttt{Clip} \left[f_i(\theta_i^t)-f_i(\hat{\theta}_i^{t+1})\right]$) with sensitivity $2C_{loss}$ and compare them with a noisy threshold $\alpha +{\rm Lap}(\frac{2 {\rm \mathbf{c}}{C}_{loss}}{\epsilon_1})$. According to Theorem 1 in \cite{lyu2017understanding}, this procedure satisfies $(\epsilon_1+\epsilon_2)$-DP and by Lemma \ref{pure-tozcdp}, it also satisfies $ \frac{(\epsilon_1+\epsilon_2)^2}{2}$-zCDP.
\end{proof}

\section{Numerical Experiments}\label{experiments}

{\bf{Datasets}.} Experiments are performed on three benchmark datasets\footnote{\url{http://archive.ics.uci.edu/ml/datasets/Adult}, \url{http://international.ipums.org}}: Adult, US, and Brazil. Adult has 48,842 data samples and 41 features, and the label is to predict whether an annual income is more than \$50k or not. US has 40,000 records and 58 features, and the label is to predict whether the annual income of an individual is more than \$25k. BR has 38,000 samples and 53 features, and the goal is to predict whether the monthly income of an individual is more than \$300.

{\bf{Data preprocessing}.} We consider the same preprocessing procedure as the method used in \cite{zhang2018improving}. We first normalize each attribute so that the maximum attribute value is 1, and normalize each sample so its $L_2$-norm at most 1. As for the label column, we also map it to $\{-1,1\}$. In each simulation, we randomly sample 35,000 records for training and divide them into $N$ parties, and thus each party includes $35000/N$ data samples (i.e., $|D_i|=35000/N$). We denote the rest of the data records as testing data. 

{\bf{Baselines}.} We compare our proposed algorithms against four baseline algorithms: (i) DVP \cite{zhang2016dynamic}, is a dual variable perturbation method, where the dual variable of each agent at each ADMM iteration is perturbed by Gamma noise. (ii) M-ADMM \cite{zhang2018improving}, is a penalty perturbation approach, where each agent's penalty variable is perturbed by Gamma noise at each ADMM iteration. (iii) R-ADMM \cite{zhang2018recycled}, is based on the penalty approach and the re-utilization of previous iteration's results to save the privacy loss. (iv) Non-private (decentralized ADMM without adding noise). Note that the privacy guarantees of DVP, M-ADMM, and R-ADMM hold only when the optimal solution of the perturbed subproblem is obtained in each iteration. In order to have a fair comparison, we adopt the Newton solver to obtain the optimal solution in each iteration. Notice that we also provide sharpened and tight privacy loss of above private ADMM algorithms under the privacy framework of zCDP.

{\bf{Setup}.} We adopt logistic loss $\mathscr{L}(y_i^n {\theta}_i^T x_i^n) =\log(1+\exp(-y_i^n \theta_i^T x_i^n))$ as loss function, and the derivative $\mathscr{L}'(\cdot)$ is bounded with $|\mathscr{L}'(\cdot)|\leq 1 $ and $c_1$-Lipschitz with $c_1=1/4$. We also let $\mathscr{R}(\theta_i) = \frac{1}{2}\|\theta_i\|_2^2$. We evaluate the accuracy by classification error rate over the testing set, defined as $Error~rate = \frac{Number~of~incorrect~predictions}{Total~number~of~predictions~made}$
and the convergence of algorithms by the average loss over the training samples, given by ~$\mathscr{L}_t  = \frac{1}{N}\sum_{i=1}^N\frac{1}{|D_i|}\sum_{n = 1}^{|D_i|}\mathscr{L}(y_i^n(\theta_i^t)^T x_i^n)$. We also report the mean and standard deviation of the average loss. The smaller the average loss, the higher accuracy.


{\bf{Parameter settings}.} We consider a randomly generated undirectedly network with $N=5$ agents and we fix the step size $\eta=0.5$ and the total iteration number $T=30$. We also consider the maximum number of primal variables that can be shared, $\mathbf{c} = 15$. Moreover, to maximize the utility of SVT, we follow the ratio between $\epsilon_1$ and $\epsilon_2$ in \cite{lyu2017understanding}, i.e., $\epsilon_1:\epsilon_2=1:(2\mathbf{c})^{\frac{2}{3}}$. In all experiments, we set $\delta=10^{-4}$, and $\epsilon = \{0.5,1,1.5,2,10\}$.

 \begin{figure} [!t]
\centering
$\begin{array}{c@{\hspace{0.0in}}c@{\hspace{0.0in}}c@{\hspace{0.0in}}c}
\includegraphics[trim={0.15cm 0 1.3cm 0.8cm},clip=true, width=1.64in]{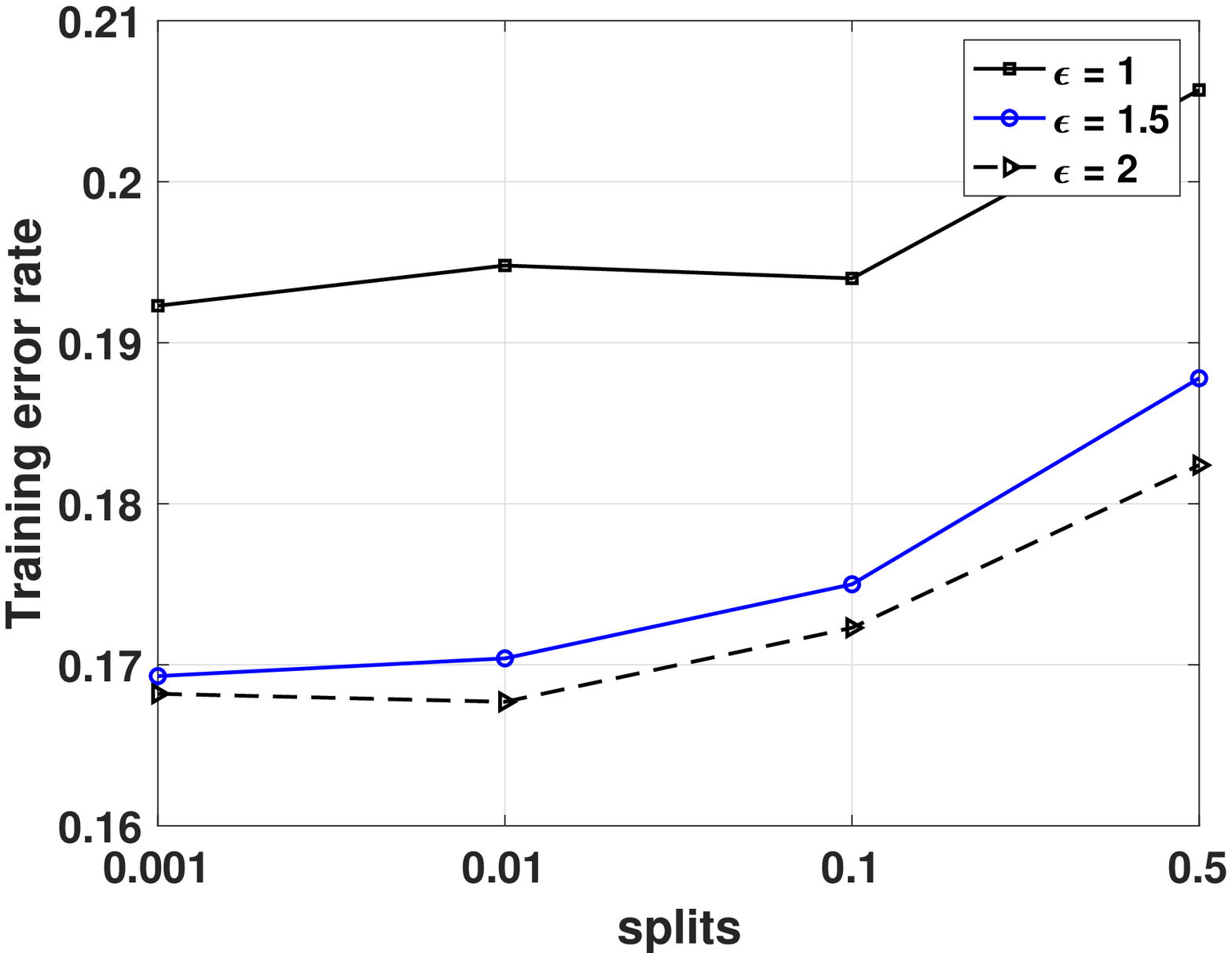} & \includegraphics[trim={0.15cm 0 1.3cm 0.8cm},clip=true, width=1.64in]{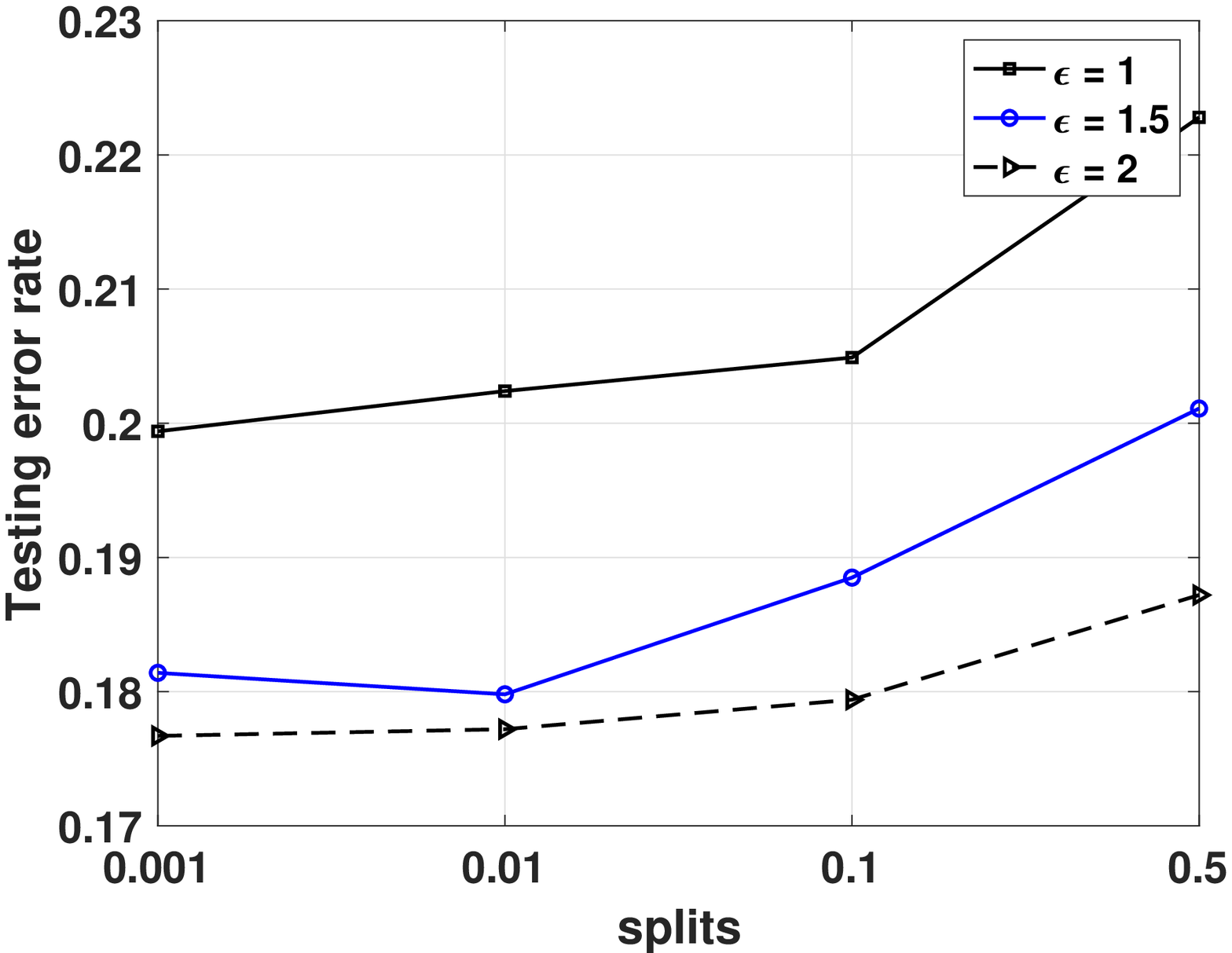} & 
\end{array}$ 
\caption{Effects of privacy budget splitting}
\label{splits}
\end{figure}

\begin{figure} [!t]
\centering
$\begin{array}{c@{\hspace{0.0in}}c@{\hspace{0.0in}}c@{\hspace{0.0in}}c}
\includegraphics[trim={0.15cm 0 1.3cm 0.8cm},clip=true, width=1.64in]{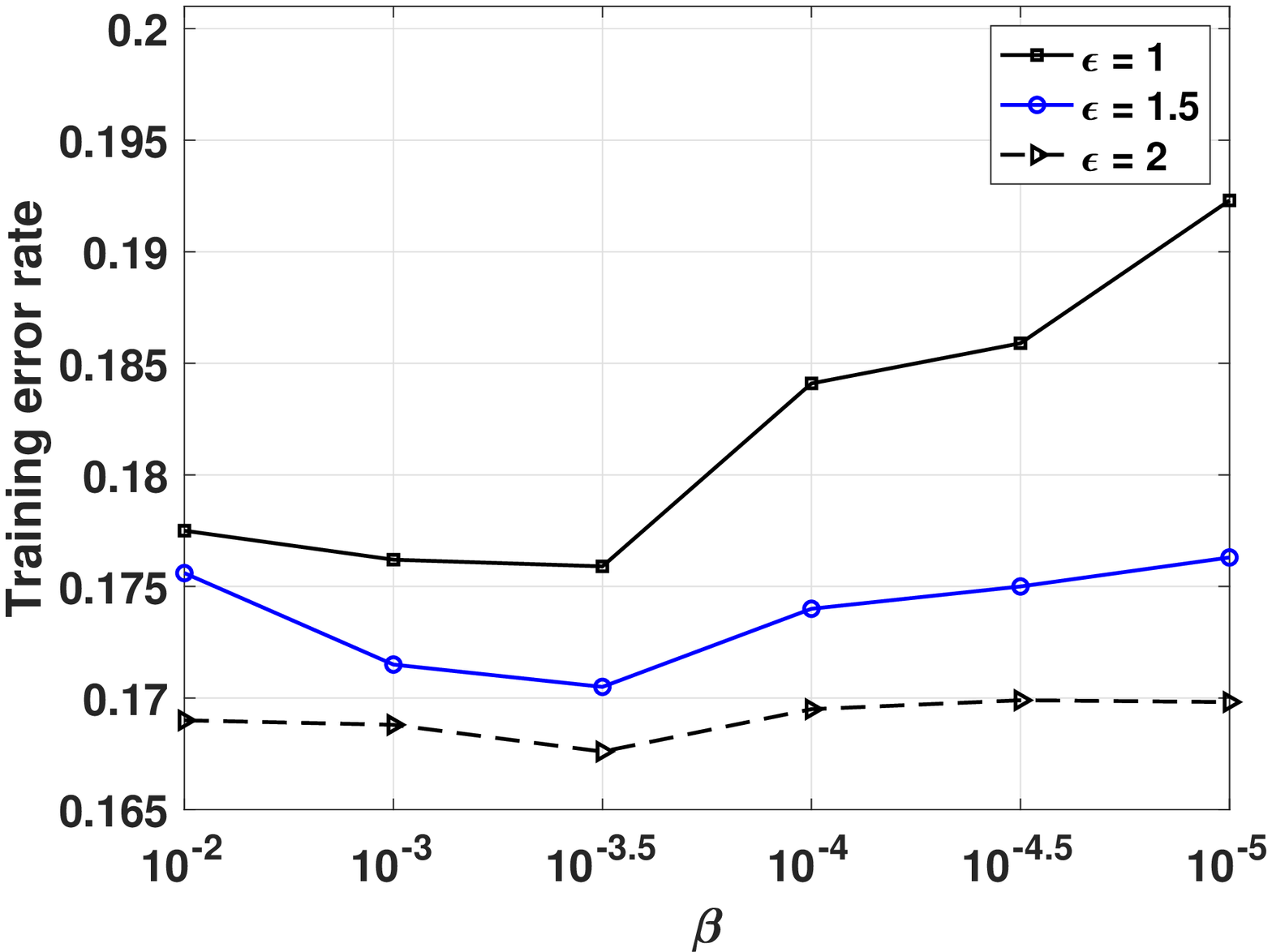} & \includegraphics[trim={0.15cm 0 1.3cm 0.8cm},clip=true, width=1.64in]{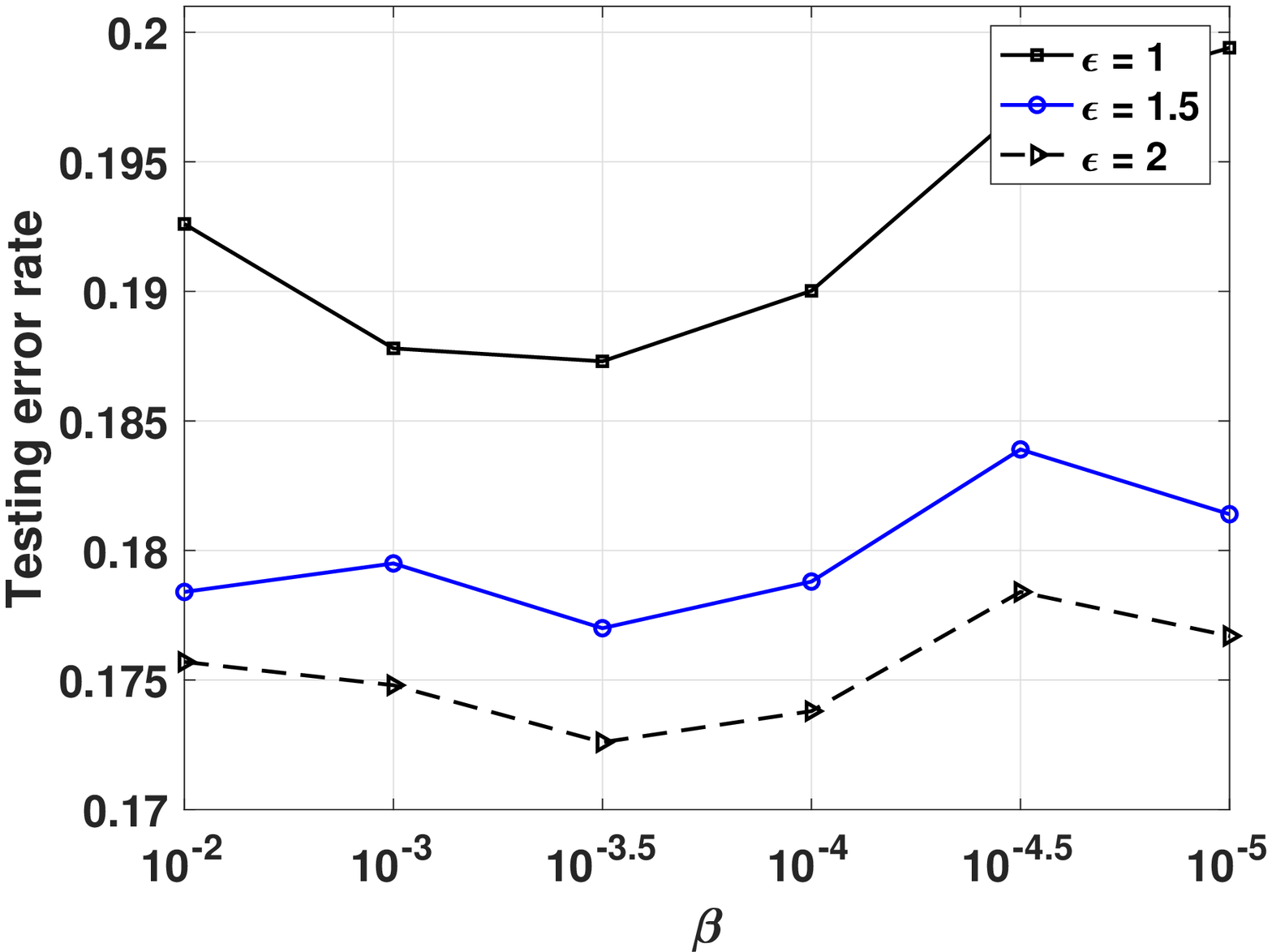} & 
\end{array}$ 
\caption{Effects of optimization accuracy $\beta$ }
\label{beta}
\end{figure}

\begin{figure} [!t]
\centering
$\begin{array}{c@{\hspace{0.0in}}c@{\hspace{0.0in}}c@{\hspace{0.0in}}c}
\includegraphics[trim={0.15cm 0 1.3cm 0.8cm},clip=true, width=1.64in]{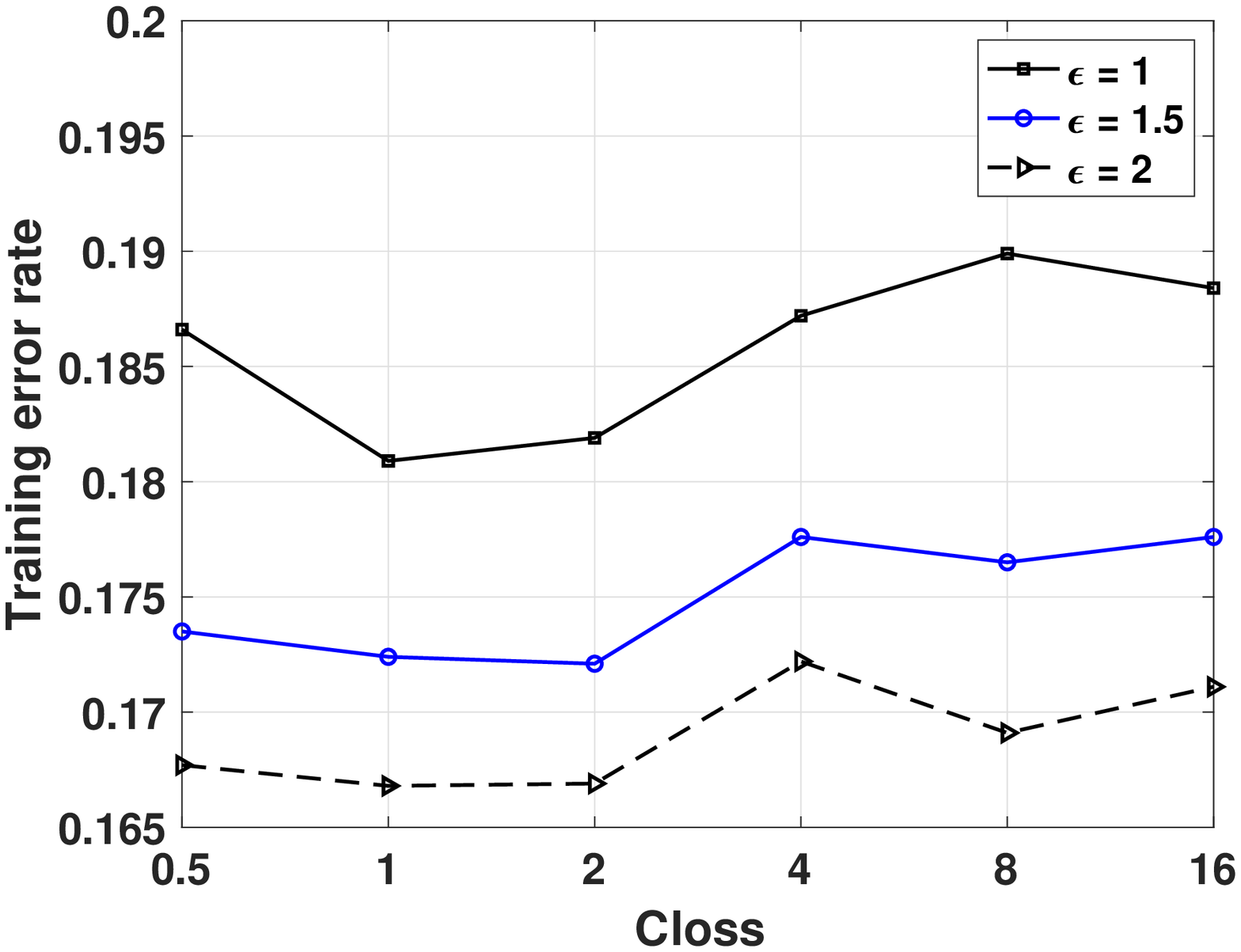}&
\includegraphics[trim={0.15cm 0 1.3cm 0.8cm},clip=true, width=1.64in]{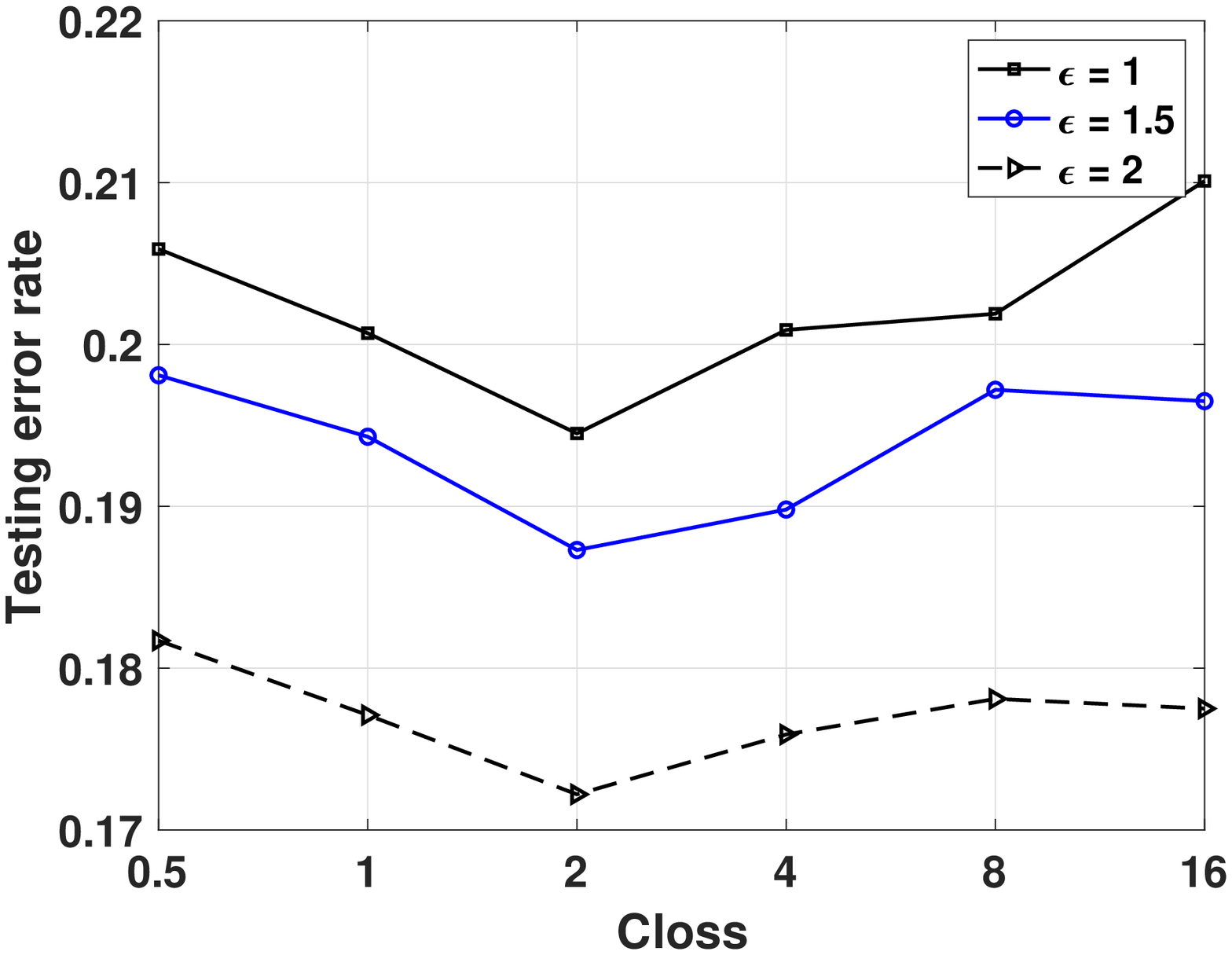} & 
\end{array}$ 
\caption{Effects of clipping threshold $C_{loss}$} 
\label{closs}
\end{figure}

\begin{figure} [!t]
\centering
$\begin{array}{c@{\hspace{0.0in}}c@{\hspace{0.0in}}c@{\hspace{0.0in}}c}
 
\includegraphics[trim={0.15cm 0 1.3cm 0.8cm},clip=true, width=1.64in]{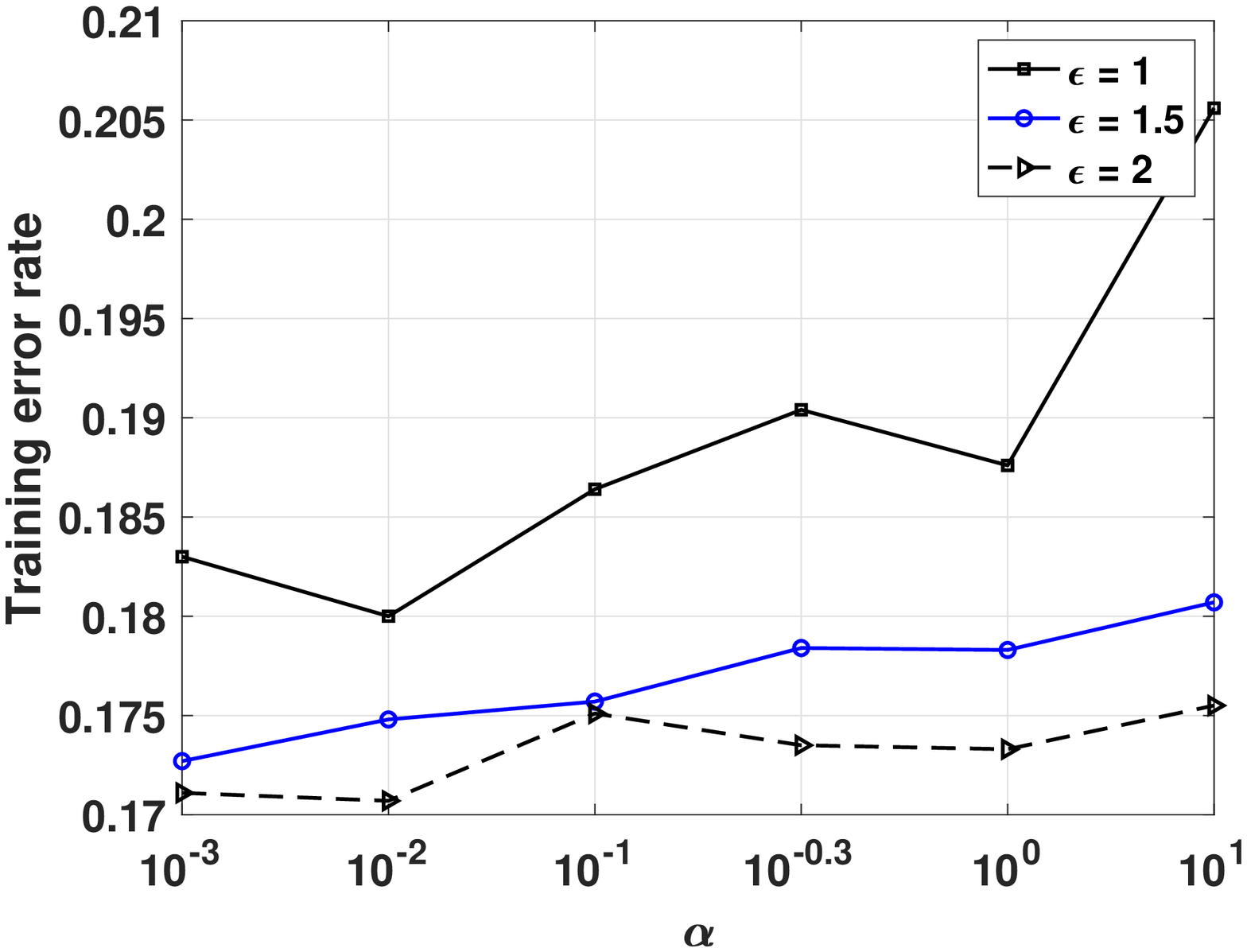}
& \includegraphics[trim={0.15cm 0 1.3cm 0.8cm},clip=true, width=1.64in]{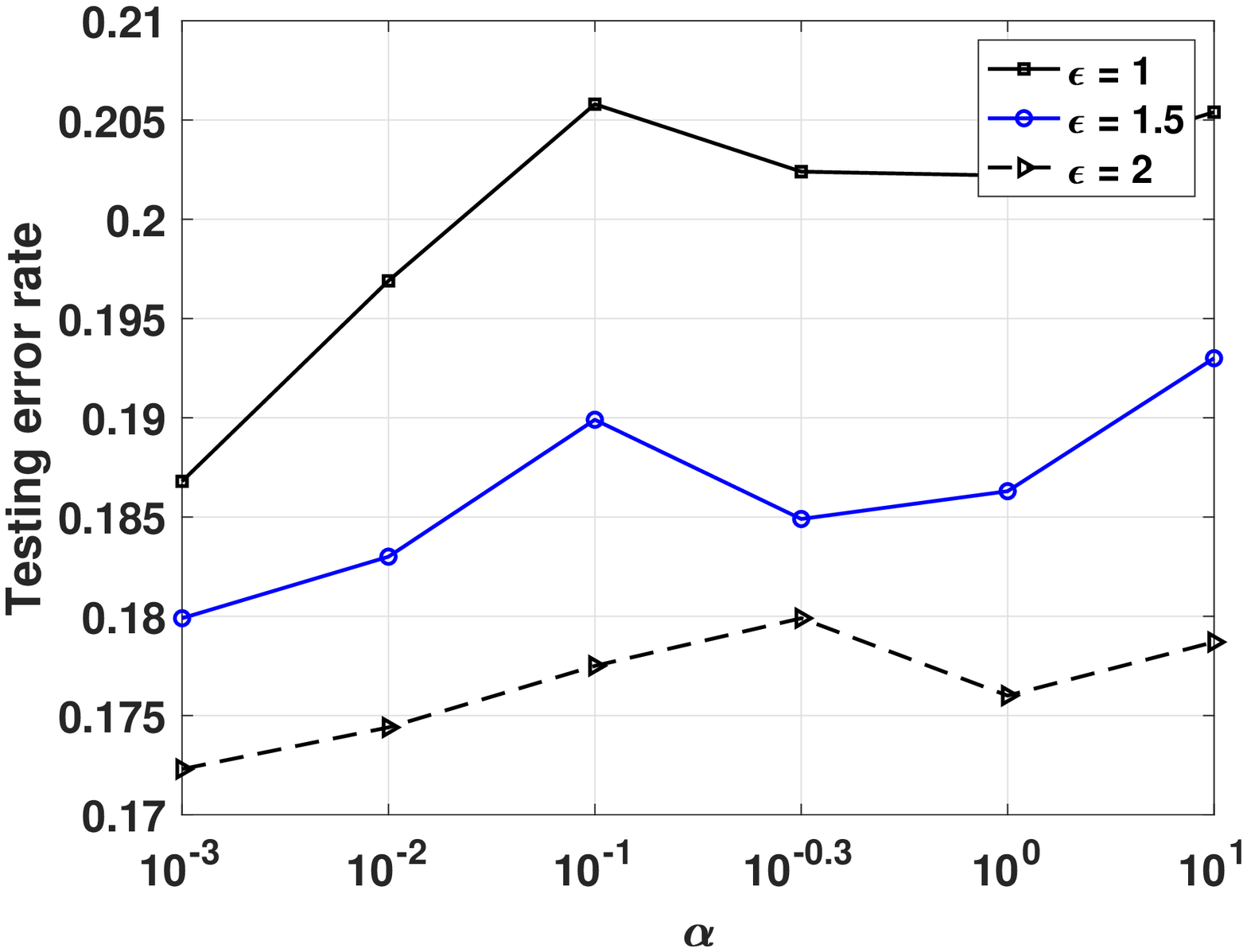} & 
\end{array}$ 
\caption{Effects of $\alpha$}
\label{alpha}
\end{figure}

\begin{figure}[t]
	\centering   
	\includegraphics[width=0.33\textwidth]{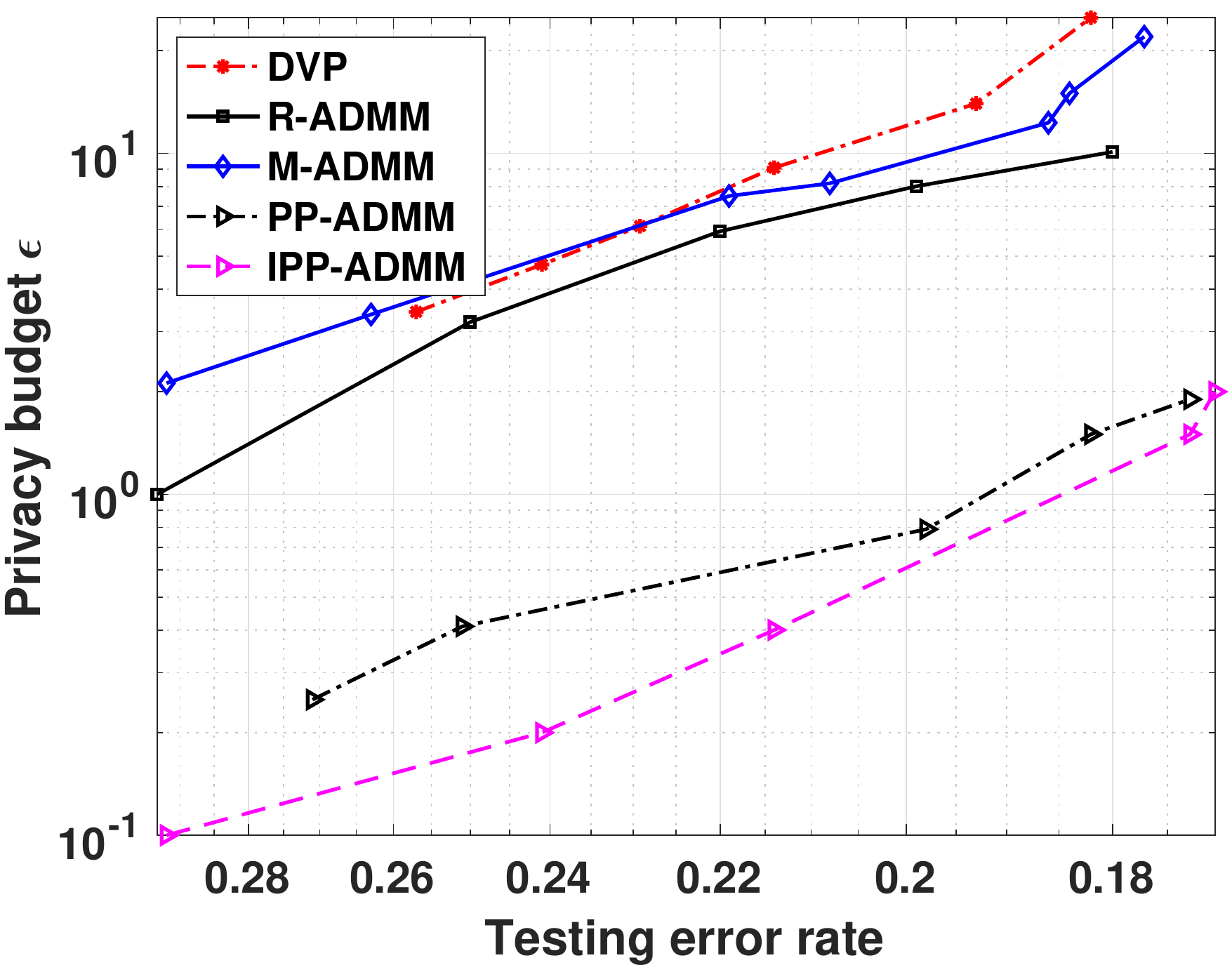}
	\caption{Trade-off between classification error rate and privacy on Adult dataset }
	\label{Cifa1}
\end{figure}

\begin{figure*}[t]
	\centering  
	\includegraphics[width=0.33\textwidth]{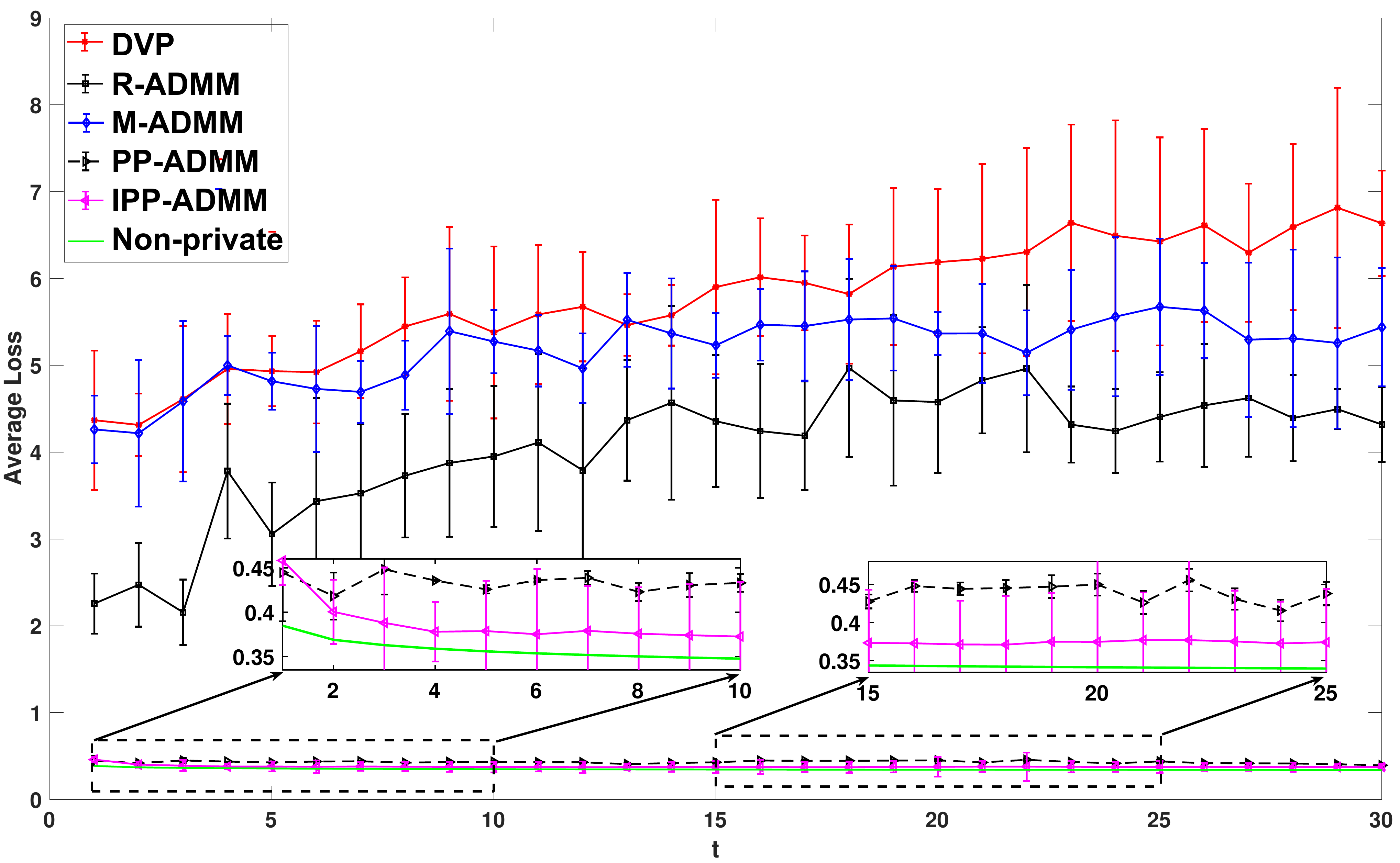}
	\includegraphics[width=0.33\textwidth]{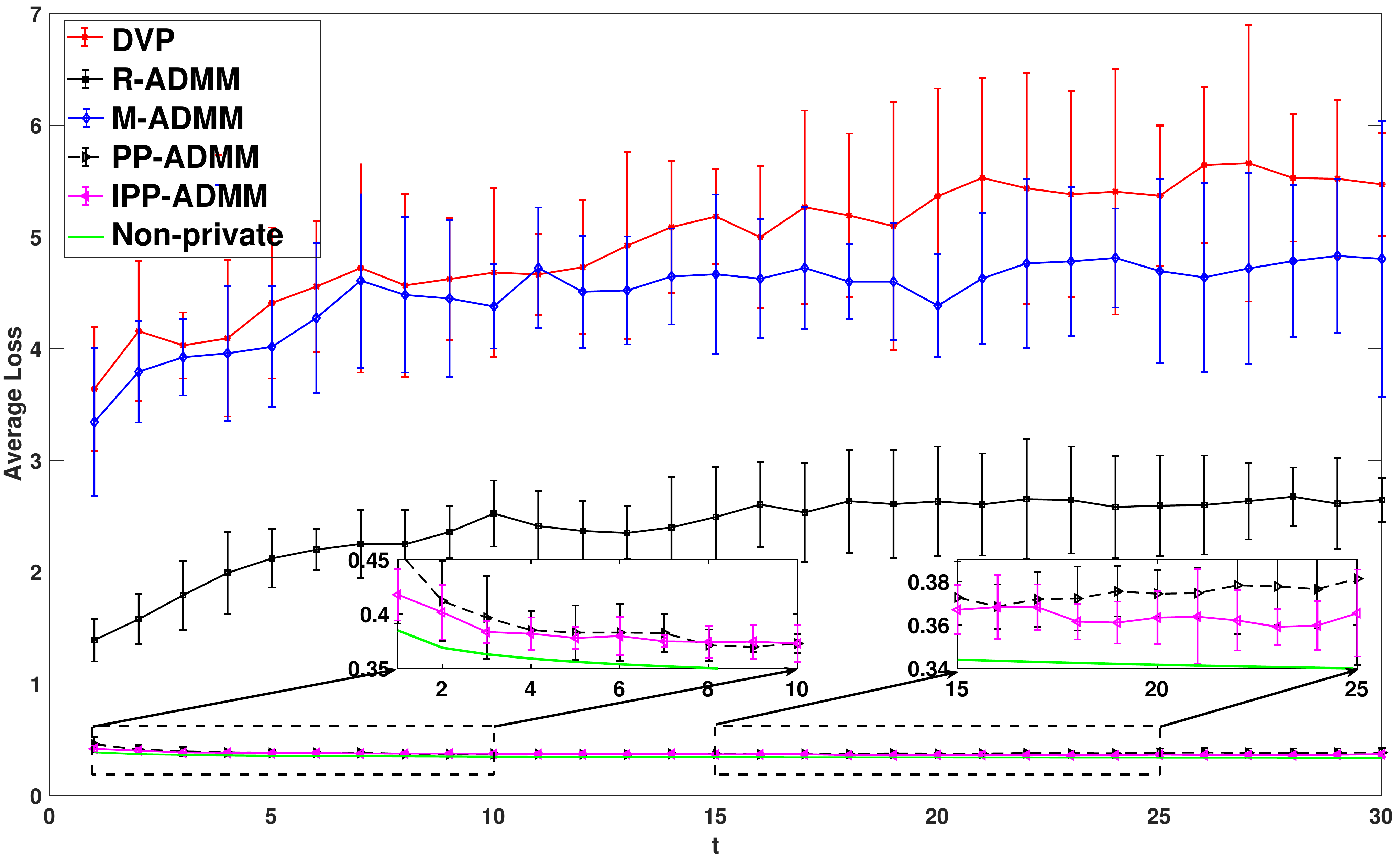}
	\includegraphics[width=0.33\textwidth]{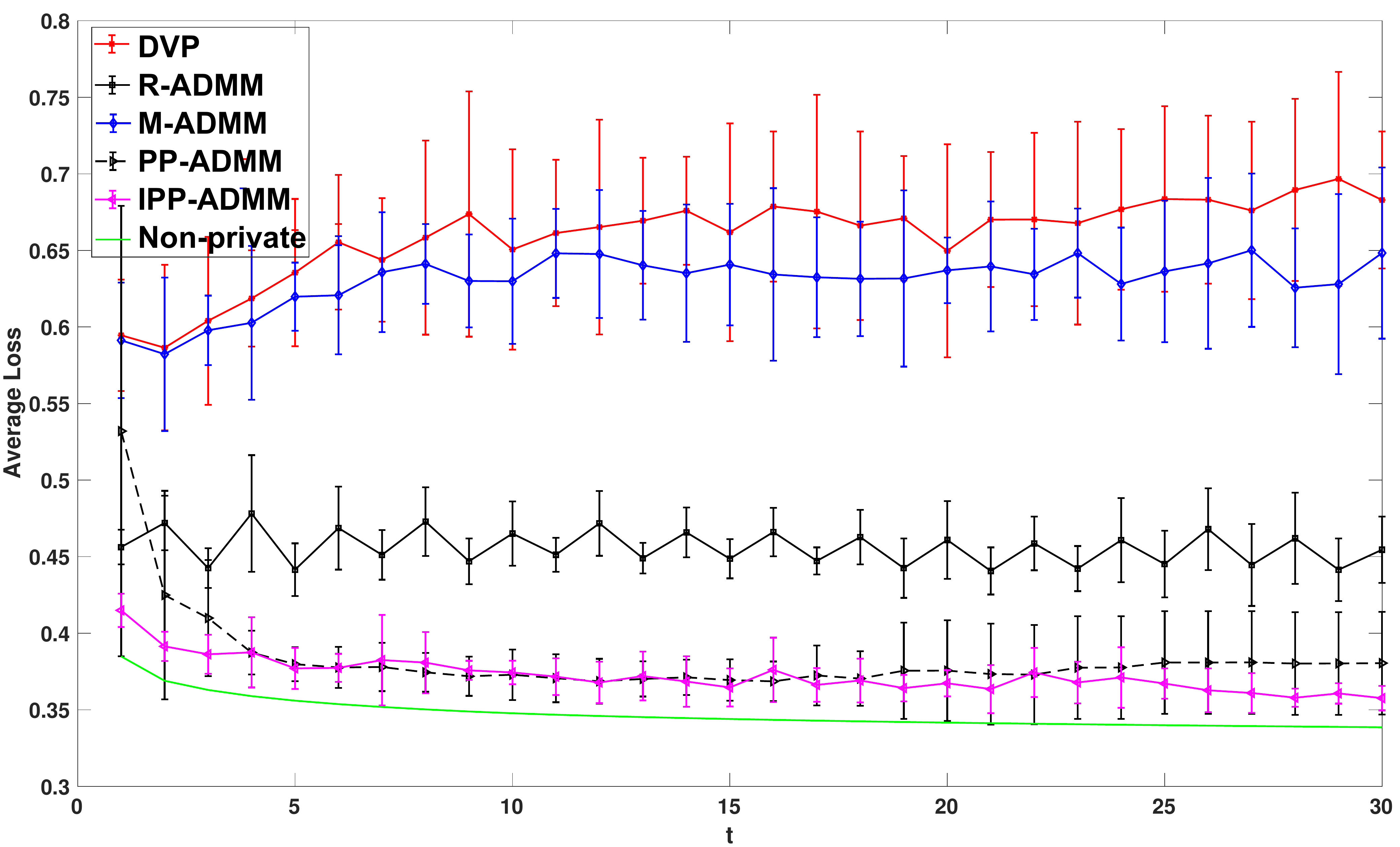}
	\caption{Convergence comparisons on Adult dataset (left: $\epsilon=1$, middle: $\epsilon=2$, right: $\epsilon=10$) }
	\label{Cifa11}
\end{figure*}
\begin{figure*} [!t]
\centering
$\begin{array}{c@{\hspace{0.5in}}c@{\hspace{0.5in}}c}
\includegraphics[trim={0.15cm 0 1.6cm 0.8cm},clip=true, width=2.5in]{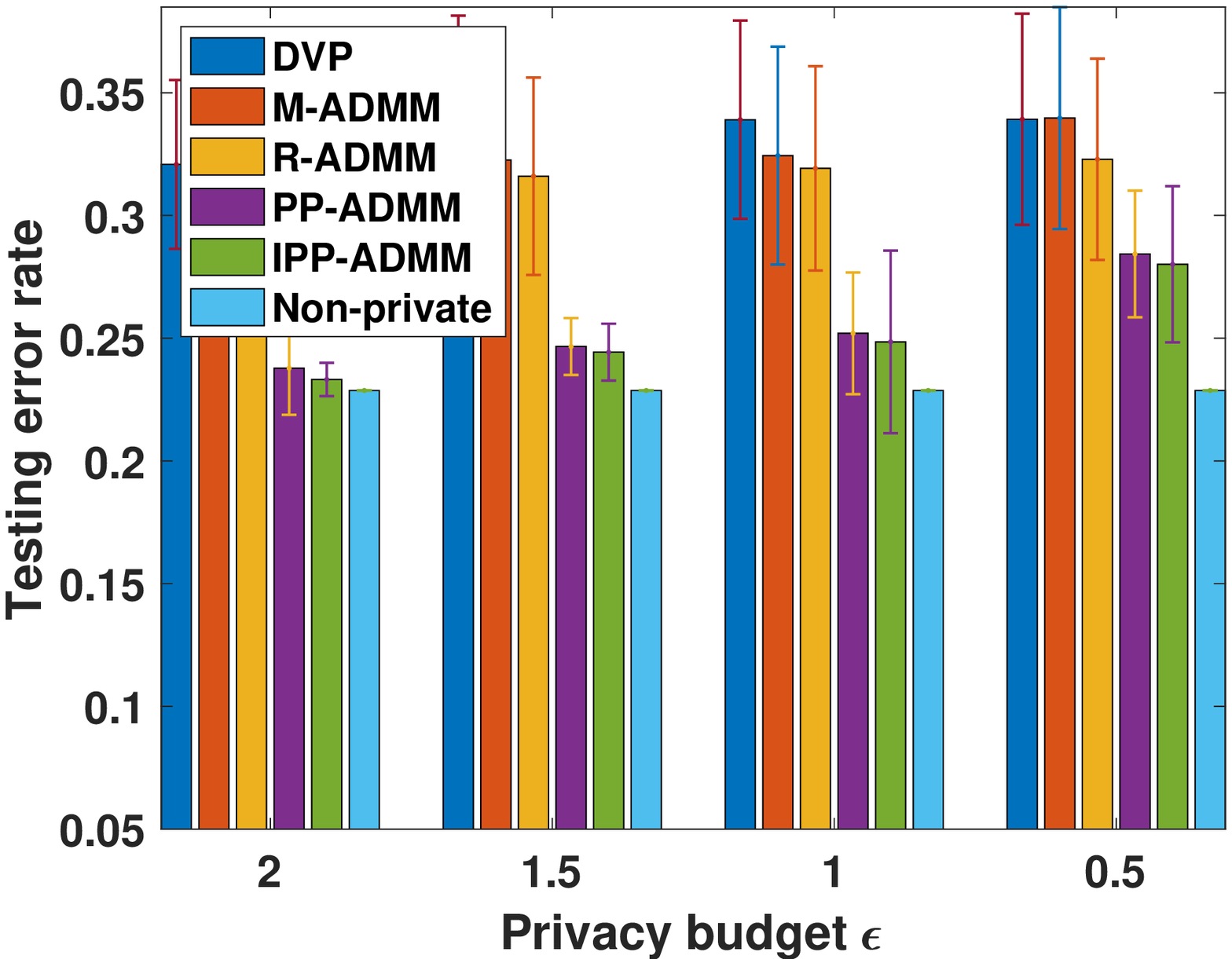} & \includegraphics[trim={0.15cm 0 1.6cm 0.8cm},clip=true, width=2.5in]{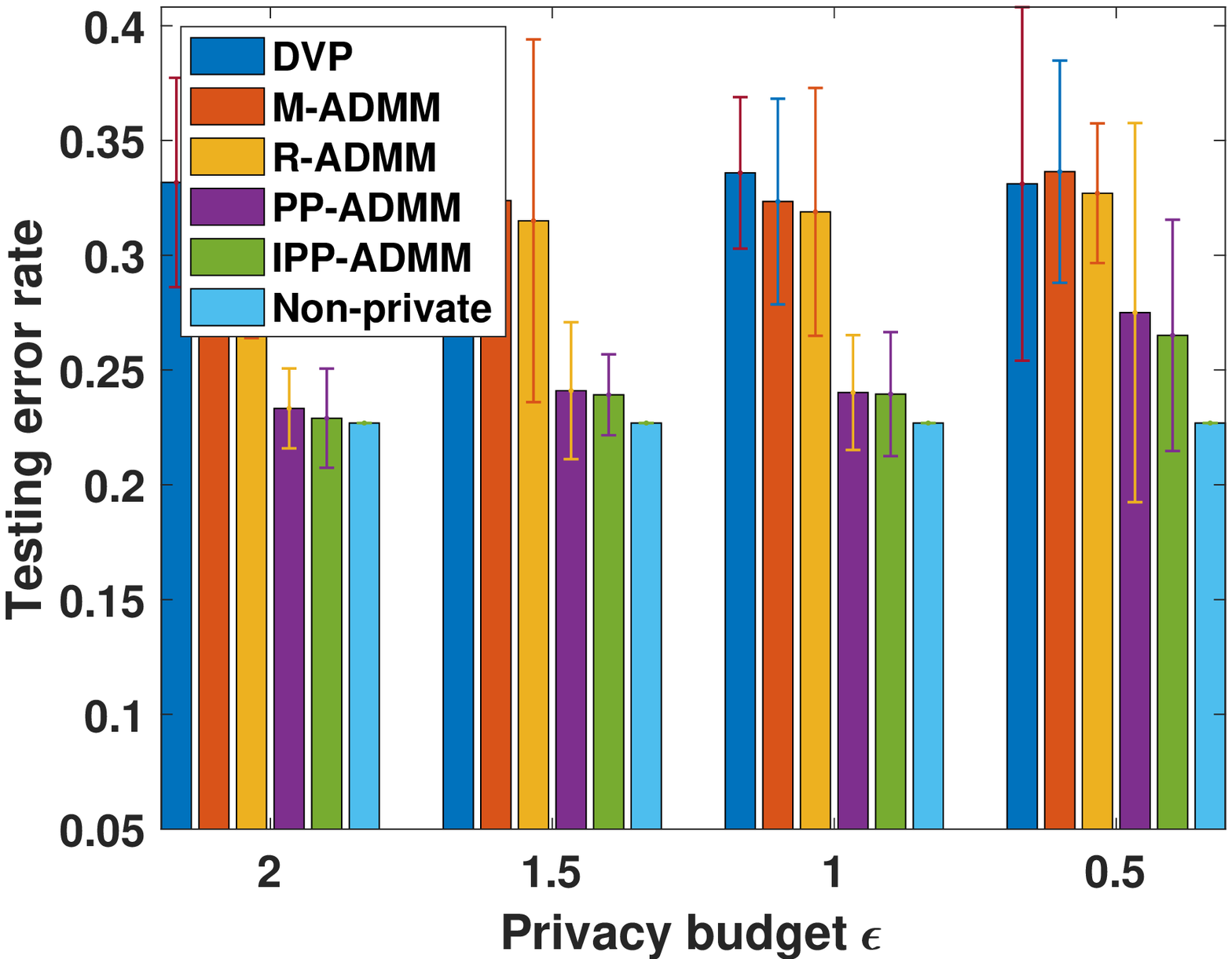}  \\ [0.0cm]
 \mbox{(a) Brazil} & \mbox{(b) US}
\end{array}$ 
\caption{Classification error rate comparisons on Brazil and US datasets.}
\label{Cifar2}
\end{figure*}

{\bf{Impacts of parameters}.} In this set of experiments on the Adult dataset, we present the effects of privacy budgets splitting and optimization accuracy (i.e., gradient norm threshold) $\beta$ on the performance of PP-ADMM, and the loss clipping threshold $C_{loss}$ and the quality function significance threshold $\alpha$ on the performance of IPP-ADMM. Specifically, we adjust different parameter settings separately, while keeping the rest constant to represent their impacts on training and testing accuracy. 

For the privacy budgets splitting of PP-ADMM, we first convert the overall privacy budget parameters $(\epsilon,\delta)$ to $\rho_{total} = \frac{\epsilon^2}{4\ln(1/\delta)}$. We set $\rho_{i1} = \frac{\rho_{total}}{T}\cdot (1-splits)$ and $\rho_{i2} =\frac{\rho_{total}}{T}\cdot {splits}$, where $splits$ denotes the fraction of $\rho_{total}$ allocated to $\rho_{i2}$. By tuning $splits$, we can find the good trade-off between the privacy budget for perturbing the objective and perturbing the approximate solution. In addition, we compute $\epsilon_{i1} = \rho_{i1}+2\sqrt{\rho_{i1}\ln(1/\delta_{i1})} $ with $\delta_{i1}= 10^{-4}$, and set $\epsilon_{i3} =0.99\cdot \epsilon_{i1}$ to dedicate most of the budget to reduce the amount of noise for perturbing the objective and increase the influence of regularization. Figure \ref{splits} shows the effects of privacy budget splitting on the performance of PP-ADMM by setting $\beta = 10^{-6}$. As $splits$ decreases, i.e., allocating less privacy budgets for perturbing the approximate solution, it yields better training and testing accuracy. Thus, we set $splits = 0.001$ to achieve a good trade-off between amount of noise added to the objective and approximate solution.

Figure \ref{beta} shows how classification accuracy changes with varying values of $\beta$ and fixing $splits = 0.001$. The parameter $\beta$ controls the optimization accuracy of each iteration of PP-ADMM training process and the amount of noise for perturbing the approximate solution. As it can be observed from the figure, due to randomness of objective introduced by the random noise, when $\beta$ is too small, solving the noisy objective perfectly in each iteration may not help the final performance. Conversely, when $\beta$ is too large, large amount of noise is added to perturb the approximate solution, which also leads to performance degradation. In our experiments, we thus fix $\beta = 10^{-3.5}$ that achieves lowest training/testing error rate. 

The IPP-ADMM algorithm has two threshold parameters, $C_{loss}$ and $\alpha$. These two parameters are used to bound the sensitivity of the quality function, and the value of quality function, respectively. If the clipping threshold $C_{loss}$ is set to a small value, it significantly reduces the sensitivity but at the same time it leads much information loss in the estimation of quality function. On the other hand, if $C_{loss}$ is large, the sensitivity becomes large that results in adding too much noise to the estimation. Thus, too large or small values of $C_{loss}$ have a negative effect on employing SVT to check whether the current approximate solution has a big enough difference from that of previous iteration. As we can see from Figure \ref{closs}, $C_{loss} = 2$ achieves a good trade-off between high information loss and large sensitivity. In Figure \ref{alpha}, we fix the the clipping threshold $C_{loss} = 2$ and vary $\alpha$ from $10^{-3}$ to $10$ to see the effect of $\alpha$ on the performance. Although large value of $\alpha$ may potentially reduce the releasing of low quality approximate solution and reduce the communication cost, we observe that it also leads the learning performance degradation. We then choose $\alpha = 10^{-3}$ in our experiments, which achieves the lowest testing/training error rate. 

{\bf{Performance comparisons}.} We also present the trade-off between classification error rate and privacy cost in Figure \ref{Cifa1}, where we measured the privacy costs of all algorithms to obtain some specified testing error rates. Figure \ref{Cifa1} illustrates that both of our methods have consistently lower privacy cost than those baselines algorithms. Compared with PP-ADMM, IPP-ADMM further saves more privacy cost due to limiting the number of releasing low-quality computational results. Additionally, we also inspect the convergence performance (i.e., average loss) of different algorithms under the same budgets, as shown in Fig. \ref{Cifa11}. We can observe that when budget $\epsilon$ decreases from 10 to 1, the average loss values of baseline algorithms increase, which matches the simulation results shown in \cite{zhang2018improving,zhang2018recycled,dingbigdata}. 
Although we also analyze the baseline algorithms using zCDP to provide tight privacy bound, using Gaussian noise instead of Gamma noise might be more beneficial to the performance, which usually has at least $\sqrt{d}$ times improvement of the empirical risk bound \cite{pmlr-v23-kifer12}, where $d$ is the dimension of training model. And our proposed algorithms continues to outperform the baseline algorithms significantly. 

Figure \ref{Cifar2} compares the accuracy (classification error rate) of different algorithms on Brazil and US. The noise parameter of all algorithms are chosen respectively so that they can achieve the same total privacy loss. As expected, the lower privacy budget, the higher classification error rate. 
As it was observed in the experiments, our proposed algorithms get close to the best achievable classification error rate for a wide range of total privacy loss considered in the experiments.
\section{Conclusions}\label{conclusion}
In this paper, we have developed (Improved) plausible differentially private ADMM algorithm, PP-ADMM and IPP-ADMM. In PP-ADMM, in order to release the shackles of the exact optimal solution during each ADMM iteration to ensure differential privacy, we consider outputting a noisy approximate solution for the perturbed objective. To further improve the utility of PP-ADMM, we have adopted SVT in IPP-ADMM to check whether the current approximate solution has a big enough difference from that of previous iteration. Moreover, we have analyzed privacy loss under the framework of zCDP and generalization performance guarantee. Finally, through the experiments on real-world datasets, we have demonstrated that the proposed algorithms outperform other differentially private ADMM based algorithms while providing the same privacy guarantee. In future work, we plan to extend our privacy analysis to non-convex loss function and apply our methods to deep learning framework. Another research direction is to study the idea of using SVT to other distributed (federated) deep learning framework to save the privacy budget and reduce the communication cost.

\begin{acks}
We thank the reviewers for their insightful comments. The work of J. Ding and M. Pan was supported in part by the U.S. National Science Foundation under grants US CNS-1646607, CNS-1801925, and CNS-2029569. The work of J. Bi was partially supported by NSF grants: CCF-1514357 and IIS-1718738, and NIH grant 5K02DA043063-03.
\end{acks}

\bibliographystyle{ACM-Reference-Format}
\bibliography{Main}


\begin{thebibliography}{26}


\ifx \showCODEN    \undefined \def \showCODEN     #1{\unskip}     \fi
\ifx \showDOI      \undefined \def \showDOI       #1{#1}\fi
\ifx \showISBNx    \undefined \def \showISBNx     #1{\unskip}     \fi
\ifx \showISBNxiii \undefined \def \showISBNxiii  #1{\unskip}     \fi
\ifx \showISSN     \undefined \def \showISSN      #1{\unskip}     \fi
\ifx \showLCCN     \undefined \def \showLCCN      #1{\unskip}     \fi
\ifx \shownote     \undefined \def \shownote      #1{#1}          \fi
\ifx \showarticletitle \undefined \def \showarticletitle #1{#1}   \fi
\ifx \showURL      \undefined \def \showURL       {\relax}        \fi
\providecommand\bibfield[2]{#2}
\providecommand\bibinfo[2]{#2}
\providecommand\natexlab[1]{#1}
\providecommand\showeprint[2][]{arXiv:#2}

\bibitem[\protect\citeauthoryear{Bun and Steinke}{Bun and Steinke}{2016}]%
        {bun2016concentrated}
\bibfield{author}{\bibinfo{person}{Mark Bun} {and} \bibinfo{person}{Thomas
  Steinke}.} \bibinfo{year}{2016}\natexlab{}.
\newblock \showarticletitle{Concentrated differential privacy: Simplifications,
  extensions, and lower bounds}. In \bibinfo{booktitle}{\emph{Theory of
  Cryptography}}. \bibinfo{publisher}{Springer Berlin Heidelberg},
  \bibinfo{address}{Beijing, China}, \bibinfo{pages}{635--658}.
\newblock


\bibitem[\protect\citeauthoryear{Chang, Hong, and Wang}{Chang
  et~al\mbox{.}}{2014}]%
        {chang2014multi}
\bibfield{author}{\bibinfo{person}{Tsung-Hui Chang}, \bibinfo{person}{Mingyi
  Hong}, {and} \bibinfo{person}{Xiangfeng Wang}.}
  \bibinfo{year}{2014}\natexlab{}.
\newblock \showarticletitle{Multi-agent distributed optimization via inexact
  consensus ADMM}.
\newblock \bibinfo{journal}{\emph{IEEE Transactions on Signal Processing}}
  \bibinfo{volume}{63}, \bibinfo{number}{2} (\bibinfo{year}{2014}),
  \bibinfo{pages}{482--497}.
\newblock


\bibitem[\protect\citeauthoryear{Chaudhuri, Monteleoni, and Sarwate}{Chaudhuri
  et~al\mbox{.}}{2011}]%
        {chaudhuri2011differentially}
\bibfield{author}{\bibinfo{person}{Kamalika Chaudhuri}, \bibinfo{person}{Claire
  Monteleoni}, {and} \bibinfo{person}{Anand~D Sarwate}.}
  \bibinfo{year}{2011}\natexlab{}.
\newblock \showarticletitle{Differentially private empirical risk
  minimization}.
\newblock \bibinfo{journal}{\emph{JMLR}} \bibinfo{volume}{12},
  \bibinfo{number}{Mar} (\bibinfo{year}{2011}), \bibinfo{pages}{1069--1109}.
\newblock


\bibitem[\protect\citeauthoryear{Ding, Errapotu, Guo, Zhang, Yuan, and
  Pan}{Ding et~al\mbox{.}}{2020a}]%
        {ding2020private}
\bibfield{author}{\bibinfo{person}{Jiahao Ding}, \bibinfo{person}{Sai~Mounika
  Errapotu}, \bibinfo{person}{Yuanxiong Guo}, \bibinfo{person}{Haixia Zhang},
  \bibinfo{person}{Dongfeng Yuan}, {and} \bibinfo{person}{Miao Pan}.}
  \bibinfo{year}{2020}\natexlab{a}.
\newblock \showarticletitle{Private Empirical Risk Minimization with Analytic
  Gaussian Mechanism for Healthcare System}.
\newblock \bibinfo{journal}{\emph{IEEE Transactions on Big Data}}
  (\bibinfo{year}{2020}), \bibinfo{pages}{1--1}.
\newblock


\bibitem[\protect\citeauthoryear{Ding, Errapotu, Zhang, Pan, and Han}{Ding
  et~al\mbox{.}}{2019a}]%
        {DingStochastic}
\bibfield{author}{\bibinfo{person}{Jiahao Ding}, \bibinfo{person}{Sai~Mounika
  Errapotu}, \bibinfo{person}{Haijun Zhang}, \bibinfo{person}{Miao Pan}, {and}
  \bibinfo{person}{Zhu Han}.} \bibinfo{year}{2019}\natexlab{a}.
\newblock \showarticletitle{Stochastic ADMM Based Distributed Machine Learning
  with Differential Privacy}. In \bibinfo{booktitle}{\emph{International
  conference on security and privacy in communication systems}}.
  \bibinfo{publisher}{Springer International Publishing},
  \bibinfo{address}{Orlando, FL}, \bibinfo{pages}{257--277}.
\newblock


\bibitem[\protect\citeauthoryear{Ding, Gong, Zhang, Pan, and Han}{Ding
  et~al\mbox{.}}{2019b}]%
        {ding2019optimal}
\bibfield{author}{\bibinfo{person}{Jiahao Ding}, \bibinfo{person}{Yanmin Gong},
  \bibinfo{person}{Chi Zhang}, \bibinfo{person}{Miao Pan}, {and}
  \bibinfo{person}{Zhu Han}.} \bibinfo{year}{2019}\natexlab{b}.
\newblock \showarticletitle{Optimal Differentially Private {ADMM} for
  Distributed Machine Learning}.
\newblock \bibinfo{journal}{\emph{CoRR}}  \bibinfo{volume}{abs/1901.02094}
  (\bibinfo{year}{2019}).
\newblock
\urldef\tempurl%
\url{https://arxiv.org/abs/1901.02094}
\showURL{%
\tempurl}


\bibitem[\protect\citeauthoryear{Ding, Zhang, Chen, Xue, Zhang, and Pan}{Ding
  et~al\mbox{.}}{2019c}]%
        {dingbigdata}
\bibfield{author}{\bibinfo{person}{Jiahao Ding}, \bibinfo{person}{Xinyue
  Zhang}, \bibinfo{person}{Mingsong Chen}, \bibinfo{person}{Kaiping Xue},
  \bibinfo{person}{Chi Zhang}, {and} \bibinfo{person}{Miao Pan}.}
  \bibinfo{year}{2019}\natexlab{c}.
\newblock \showarticletitle{Differentially Private Robust ADMM for Distributed
  Machine Learning}. In \bibinfo{booktitle}{\emph{IEEE International Conference
  on Big Data}}. \bibinfo{publisher}{IEEE}, \bibinfo{address}{Los Angeles, CA},
  \bibinfo{pages}{1302--1311}.
\newblock


\bibitem[\protect\citeauthoryear{Ding, Zhang, Li, Wang, Yu, and Pan}{Ding
  et~al\mbox{.}}{2020b}]%
        {ding2020differentially}
\bibfield{author}{\bibinfo{person}{Jiahao Ding}, \bibinfo{person}{Xinyue
  Zhang}, \bibinfo{person}{Xiaohuan Li}, \bibinfo{person}{Junyi Wang},
  \bibinfo{person}{Rong Yu}, {and} \bibinfo{person}{Miao Pan}.}
  \bibinfo{year}{2020}\natexlab{b}.
\newblock \showarticletitle{Differentially Private and Fair Classification via
  Calibrated Functional Mechanism}. In \bibinfo{booktitle}{\emph{Proceedings of
  the AAAI Conference on Artificial Intelligence}}. \bibinfo{publisher}{AAAI},
  \bibinfo{address}{New York, NY}, \bibinfo{pages}{622--629}.
\newblock


\bibitem[\protect\citeauthoryear{Dwork, McSherry, Nissim, and Smith}{Dwork
  et~al\mbox{.}}{2006}]%
        {dwork2006calibrating}
\bibfield{author}{\bibinfo{person}{Cynthia Dwork}, \bibinfo{person}{Frank
  McSherry}, \bibinfo{person}{Kobbi Nissim}, {and} \bibinfo{person}{Adam
  Smith}.} \bibinfo{year}{2006}\natexlab{}.
\newblock \showarticletitle{Calibrating noise to sensitivity in private data
  analysis}. In \bibinfo{booktitle}{\emph{Theory of cryptography}}.
  \bibinfo{publisher}{Springer Berlin Heidelberg}, \bibinfo{address}{Berlin,
  Heidelberg}, \bibinfo{pages}{265--284}.
\newblock


\bibitem[\protect\citeauthoryear{Dwork, Naor, Reingold, Rothblum, and
  Vadhan}{Dwork et~al\mbox{.}}{2009}]%
        {dwork2009complexity}
\bibfield{author}{\bibinfo{person}{Cynthia Dwork}, \bibinfo{person}{Moni Naor},
  \bibinfo{person}{Omer Reingold}, \bibinfo{person}{Guy~N Rothblum}, {and}
  \bibinfo{person}{Salil Vadhan}.} \bibinfo{year}{2009}\natexlab{}.
\newblock \showarticletitle{On the complexity of differentially private data
  release: efficient algorithms and hardness results}. In
  \bibinfo{booktitle}{\emph{Proceedings of the forty-first annual ACM symposium
  on Theory of computing}}. \bibinfo{publisher}{ACM}, \bibinfo{address}{New
  York, NY}, \bibinfo{pages}{381--390}.
\newblock


\bibitem[\protect\citeauthoryear{Dwork and Roth}{Dwork and Roth}{2014}]%
        {dwork2014algorithmic}
\bibfield{author}{\bibinfo{person}{Cynthia Dwork} {and} \bibinfo{person}{Aaron
  Roth}.} \bibinfo{year}{2014}\natexlab{}.
\newblock \bibinfo{booktitle}{\emph{The algorithmic foundations of differential
  privacy} (\bibinfo{edition}{1st} ed.)}.
\newblock \bibinfo{publisher}{Now Publishers, Inc.}, \bibinfo{address}{Boston,
  MA, USA}.
\newblock
\showISBNx{9781601988188}


\bibitem[\protect\citeauthoryear{Huang, Hu, Guo, Chan-Tin, and Gong}{Huang
  et~al\mbox{.}}{2019}]%
        {huang2018dp}
\bibfield{author}{\bibinfo{person}{Zonghao Huang}, \bibinfo{person}{Rui Hu},
  \bibinfo{person}{Yuanxiong Guo}, \bibinfo{person}{Eric Chan-Tin}, {and}
  \bibinfo{person}{Yanmin Gong}.} \bibinfo{year}{2019}\natexlab{}.
\newblock \showarticletitle{DP-ADMM: ADMM-based distributed learning with
  differential privacy}.
\newblock \bibinfo{journal}{\emph{IEEE Transactions on Information Forensics
  and Security}}  \bibinfo{volume}{15} (\bibinfo{date}{July}
  \bibinfo{year}{2019}), \bibinfo{pages}{1002--1012}.
\newblock


\bibitem[\protect\citeauthoryear{Kifer, Smith, and Thakurta}{Kifer
  et~al\mbox{.}}{2012}]%
        {pmlr-v23-kifer12}
\bibfield{author}{\bibinfo{person}{Daniel Kifer}, \bibinfo{person}{Adam Smith},
  {and} \bibinfo{person}{Abhradeep Thakurta}.} \bibinfo{year}{2012}\natexlab{}.
\newblock \showarticletitle{Private Convex Empirical Risk Minimization and
  High-dimensional Regression}. In \bibinfo{booktitle}{\emph{Proceedings of the
  25th Annual Conference on Learning Theory}}, Vol.~\bibinfo{volume}{23}.
  \bibinfo{publisher}{PMLR}, \bibinfo{address}{Edinburgh, Scotland},
  \bibinfo{pages}{25.1--25.40}.
\newblock


\bibitem[\protect\citeauthoryear{Lyu, Su, and Li}{Lyu et~al\mbox{.}}{2017}]%
        {lyu2017understanding}
\bibfield{author}{\bibinfo{person}{Min Lyu}, \bibinfo{person}{Dong Su}, {and}
  \bibinfo{person}{Ninghui Li}.} \bibinfo{year}{2017}\natexlab{}.
\newblock \showarticletitle{Understanding the sparse vector technique for
  differential privacy}.
\newblock \bibinfo{journal}{\emph{Proceedings of the VLDB Endowment}}
  \bibinfo{volume}{10}, \bibinfo{number}{6} (\bibinfo{year}{2017}),
  \bibinfo{pages}{637--648}.
\newblock


\bibitem[\protect\citeauthoryear{Mateos, Bazerque, and Giannakis}{Mateos
  et~al\mbox{.}}{2010}]%
        {mateos2010distributed}
\bibfield{author}{\bibinfo{person}{Gonzalo Mateos},
  \bibinfo{person}{Juan~Andr{\'e}s Bazerque}, {and} \bibinfo{person}{Georgios~B
  Giannakis}.} \bibinfo{year}{2010}\natexlab{}.
\newblock \showarticletitle{Distributed sparse linear regression}.
\newblock \bibinfo{journal}{\emph{IEEE Transactions on Signal Processing}}
  \bibinfo{volume}{58}, \bibinfo{number}{10} (\bibinfo{year}{2010}),
  \bibinfo{pages}{5262--5276}.
\newblock


\bibitem[\protect\citeauthoryear{Melis, Song, De~Cristofaro, and
  Shmatikov}{Melis et~al\mbox{.}}{2019}]%
        {melis2019exploiting}
\bibfield{author}{\bibinfo{person}{Luca Melis}, \bibinfo{person}{Congzheng
  Song}, \bibinfo{person}{Emiliano De~Cristofaro}, {and}
  \bibinfo{person}{Vitaly Shmatikov}.} \bibinfo{year}{2019}\natexlab{}.
\newblock \showarticletitle{Exploiting unintended feature leakage in
  collaborative learning}. In \bibinfo{booktitle}{\emph{2019 IEEE S\&P}}.
  \bibinfo{publisher}{IEEE}, \bibinfo{address}{San Francisco, CA},
  \bibinfo{pages}{691--706}.
\newblock


\bibitem[\protect\citeauthoryear{Mota, Xavier, Aguiar, and P{\"u}schel}{Mota
  et~al\mbox{.}}{2013}]%
        {mota2013d}
\bibfield{author}{\bibinfo{person}{Joao~FC Mota}, \bibinfo{person}{Joao~MF
  Xavier}, \bibinfo{person}{Pedro~MQ Aguiar}, {and} \bibinfo{person}{Markus
  P{\"u}schel}.} \bibinfo{year}{2013}\natexlab{}.
\newblock \showarticletitle{D-ADMM: A communication-efficient distributed
  algorithm for separable optimization}.
\newblock \bibinfo{journal}{\emph{IEEE Transactions on Signal Processing}}
  \bibinfo{volume}{61}, \bibinfo{number}{10} (\bibinfo{year}{2013}),
  \bibinfo{pages}{2718--2723}.
\newblock


\bibitem[\protect\citeauthoryear{Shi, Ding, Errapotu, Yue, Xu, Zhou, and
  Pan}{Shi et~al\mbox{.}}{2019}]%
        {shi2019deep}
\bibfield{author}{\bibinfo{person}{Dian Shi}, \bibinfo{person}{Jiahao Ding},
  \bibinfo{person}{Sai~Mounika Errapotu}, \bibinfo{person}{Hao Yue},
  \bibinfo{person}{Wenjun Xu}, \bibinfo{person}{Xiangwei Zhou}, {and}
  \bibinfo{person}{Miao Pan}.} \bibinfo{year}{2019}\natexlab{}.
\newblock \showarticletitle{Deep Q-Network-Based Route Scheduling for TNC
  Vehicles With Passengers’ Location Differential Privacy}.
\newblock \bibinfo{journal}{\emph{IEEE Internet of Things Journal}}
  \bibinfo{volume}{6}, \bibinfo{number}{5} (\bibinfo{year}{2019}),
  \bibinfo{pages}{7681--7692}.
\newblock


\bibitem[\protect\citeauthoryear{Shokri, Stronati, Song, and Shmatikov}{Shokri
  et~al\mbox{.}}{2017}]%
        {shokri2017membership}
\bibfield{author}{\bibinfo{person}{Reza Shokri}, \bibinfo{person}{Marco
  Stronati}, \bibinfo{person}{Congzheng Song}, {and} \bibinfo{person}{Vitaly
  Shmatikov}.} \bibinfo{year}{2017}\natexlab{}.
\newblock \showarticletitle{Membership inference attacks against machine
  learning models}. In \bibinfo{booktitle}{\emph{2017 IEEE S\&P}}.
  \bibinfo{publisher}{IEEE}, \bibinfo{address}{San Jose, CA},
  \bibinfo{pages}{3--18}.
\newblock


\bibitem[\protect\citeauthoryear{Wang, Zhang, Zhang, Lin, Tode, Pan, and
  Han}{Wang et~al\mbox{.}}{2018}]%
        {wang2018globe}
\bibfield{author}{\bibinfo{person}{Jingyi Wang}, \bibinfo{person}{Xinyue
  Zhang}, \bibinfo{person}{Haijun Zhang}, \bibinfo{person}{Hai Lin},
  \bibinfo{person}{Hideki Tode}, \bibinfo{person}{Miao Pan}, {and}
  \bibinfo{person}{Zhu Han}.} \bibinfo{year}{2018}\natexlab{}.
\newblock \showarticletitle{Data-Driven Optimization for Utility Providers with
  Differential Privacy of Users' Energy Profile}. In
  \bibinfo{booktitle}{\emph{2018 IEEE Global Communications Conference
  (GLOBECOM)}}. \bibinfo{publisher}{IEEE}, \bibinfo{address}{Singapore},
  \bibinfo{pages}{1--6}.
\newblock


\bibitem[\protect\citeauthoryear{{Yu}, {Liu}, {Pu}, {Gursoy}, and {Truex}}{{Yu}
  et~al\mbox{.}}{2019}]%
        {8835283}
\bibfield{author}{\bibinfo{person}{L. {Yu}}, \bibinfo{person}{L. {Liu}},
  \bibinfo{person}{C. {Pu}}, \bibinfo{person}{M.~E. {Gursoy}}, {and}
  \bibinfo{person}{S. {Truex}}.} \bibinfo{year}{2019}\natexlab{}.
\newblock \showarticletitle{Differentially Private Model Publishing for Deep
  Learning}. In \bibinfo{booktitle}{\emph{2019 IEEE S\&P}}.
  \bibinfo{publisher}{IEEE Computer Society}, \bibinfo{address}{Los Alamitos,
  CA, USA}, \bibinfo{pages}{332--349}.
\newblock


\bibitem[\protect\citeauthoryear{Zhang and Kwok}{Zhang and Kwok}{2014}]%
        {zhang2014asynchronous}
\bibfield{author}{\bibinfo{person}{Ruiliang Zhang} {and} \bibinfo{person}{James
  Kwok}.} \bibinfo{year}{2014}\natexlab{}.
\newblock \showarticletitle{Asynchronous distributed ADMM for consensus
  optimization}. In \bibinfo{booktitle}{\emph{ICML}}.
  \bibinfo{publisher}{JMLR}, \bibinfo{address}{Beijing, China},
  \bibinfo{pages}{1701--1709}.
\newblock


\bibitem[\protect\citeauthoryear{Zhang and Zhu}{Zhang and Zhu}{2016}]%
        {zhang2016dynamic}
\bibfield{author}{\bibinfo{person}{Tao Zhang} {and} \bibinfo{person}{Quanyan
  Zhu}.} \bibinfo{year}{2016}\natexlab{}.
\newblock \showarticletitle{Dynamic differential privacy for ADMM-based
  distributed classification learning}.
\newblock \bibinfo{journal}{\emph{IEEE Transactions on Information Forensics
  and Security}} \bibinfo{volume}{12}, \bibinfo{number}{1}
  (\bibinfo{year}{2016}), \bibinfo{pages}{172--187}.
\newblock


\bibitem[\protect\citeauthoryear{Zhang, Ding, Errapotu, Huang, Li, and
  Pan}{Zhang et~al\mbox{.}}{2019}]%
        {zhang2019differentially}
\bibfield{author}{\bibinfo{person}{Xinyue Zhang}, \bibinfo{person}{Jiahao
  Ding}, \bibinfo{person}{Sai~Mounika Errapotu}, \bibinfo{person}{Xiaoxia
  Huang}, \bibinfo{person}{Pan Li}, {and} \bibinfo{person}{Miao Pan}.}
  \bibinfo{year}{2019}\natexlab{}.
\newblock \showarticletitle{Differentially Private Functional Mechanism for
  Generative Adversarial Networks}. In \bibinfo{booktitle}{\emph{2019 IEEE
  Global Communications Conference (GLOBECOM)}}. \bibinfo{publisher}{IEEE},
  \bibinfo{address}{Waikoloa, HI}, \bibinfo{pages}{1--6}.
\newblock


\bibitem[\protect\citeauthoryear{Zhang, Khalili, and Liu}{Zhang
  et~al\mbox{.}}{2018a}]%
        {zhang2018improving}
\bibfield{author}{\bibinfo{person}{Xueru Zhang},
  \bibinfo{person}{Mohammad~Mahdi Khalili}, {and} \bibinfo{person}{Mingyan
  Liu}.} \bibinfo{year}{2018}\natexlab{a}.
\newblock \showarticletitle{Improving the Privacy and Accuracy of {ADMM}-Based
  Distributed Algorithms}. In \bibinfo{booktitle}{\emph{ICML}}.
  \bibinfo{publisher}{PMLR}, \bibinfo{address}{Stockholmsmässan, Stockholm
  Sweden}, \bibinfo{pages}{5796--5805}.
\newblock


\bibitem[\protect\citeauthoryear{Zhang, Khalili, and Liu}{Zhang
  et~al\mbox{.}}{2018b}]%
        {zhang2018recycled}
\bibfield{author}{\bibinfo{person}{Xueru Zhang},
  \bibinfo{person}{Mohammad~Mahdi Khalili}, {and} \bibinfo{person}{Mingyan
  Liu}.} \bibinfo{year}{2018}\natexlab{b}.
\newblock \showarticletitle{Recycled admm: Improve privacy and accuracy with
  less computation in distributed algorithms}. In
  \bibinfo{booktitle}{\emph{Annual Allerton Conference on Communication,
  Control, and Computing (Allerton)}}. \bibinfo{publisher}{IEEE},
  \bibinfo{address}{Monticello, IL}, \bibinfo{pages}{959--965}.
\newblock


\end{thebibliography}


\end{document}